\documentclass{article}

\usepackage[final]{neurips_2022}

\usepackage[utf8]{inputenc} % allow utf-8 input
\usepackage[T1]{fontenc}    % use 8-bit T1 fonts
\usepackage{hyperref}       % hyperlinks
\usepackage{url}            % simple URL typesetting
\usepackage{booktabs}       % professional-quality tables
\usepackage{amsfonts}       % blackboard math symbols
\usepackage{nicefrac}       % compact symbols for 1/2, etc.
\usepackage{microtype}      % microtypography
\usepackage{xcolor}         % colors

\title{Differentially Private Online-to-Batch for Smooth Losses}

\author{%
  Qinzi Zhang\\
  Dept. of Electrical and Computer Engineering\\
  Boston University\\
  \texttt{qinziz@bu.edu}\\
   \And
  Hoang Tran\\
  Dept. of Electrical and Computer Engineering\\
  Boston University\\
  \texttt{tranhp@bu.edu}\\
  \AND
  Ashok Cutkosky\\
  Dept. of Electrical and Computer Engineering\\
  Boston University\\
  \texttt{ashok@cutkosky.com}\\
}

\usepackage{amsmath, amssymb, amsthm}
\usepackage{algorithm}
\usepackage[noend]{algpseudocode}
\usepackage{thmtools,thm-restate}

\newtheorem{theorem}{Theorem}
\newtheorem{lemma}[theorem]{Lemma}

\newtheorem{proposition}[theorem]{Proposition}

\theoremstyle{definition}
\newtheorem{definition}{Definition}
\newtheorem{assumption}{Assumption}
\newtheorem{remark}[theorem]{Remark}

\newcommand{\RR}{\mathbb{R}}

\newcommand{\A}{\mathcal{A}}

\renewcommand{\L}{\mathcal{L}}
\newcommand{\D}{\mathcal{D}}

\renewcommand{\P}{\mathop{\mathcal{P}}}
\newcommand{\E}{\mathop{\mathbb{E}}}

\newcommand{\R}{\mathbb{R}}
\newcommand{\argmin}{\mathop{\text{argmin}}}
\newcommand{\argmax}{\mathop{\text{argmax}}}

\newcommand{\N}{\mathcal{N}}

\newcommand{\IN}{\textsc{IN}}
\newcommand{\OUT}{\textsc{OUT}}

\newcommand{\Z}{\mathcal{Z}}

\newcommand{\nSG}{\mathrm{nSG}}
\newcommand{\reg}{\mathrm{Regret}}
\newcommand{\K}{\mathcal{K}}

\newcommand{\weight}{\beta}

\begin{document}

\maketitle

\begin{abstract}
  We develop a new reduction that converts any online convex optimization algorithm suffering $O(\sqrt{T})$ regret into an $\epsilon$-differentially private stochastic convex optimization algorithm with the optimal convergence rate $\tilde O(1/\sqrt{T} + \sqrt{d}/\epsilon T)$ on smooth losses in linear time, forming a direct analogy to the classical non-private ``online-to-batch'' conversion. 
%   In the $\mu$-strongly-convex case, the same argument converts algorithms achieving $O(\log(T)/\mu)$ regret into stochastic optimization algorithms achieving a convergence rate $\tilde O(1/\mu T + 1/\mu\epsilon^2T^2)$. 
By applying our techniques to more advanced adaptive online algorithms, we produce adaptive differentially private counterparts whose convergence rates depend on apriori unknown variances or parameter norms.
\end{abstract}

\section{Introduction} \label{sec:intro}

Solving stochastic convex optimization (SCO) problems forms a core component of training models in machine learning, and is the topic of this paper. The SCO problem is to optimize an objective $\mathcal{L}$:
\begin{align}
    \min_{x\in W} \L(x) = \E_{z\sim P_z}[\ell(x,z)]\label{eqn:stochopt}
\end{align}
Here, $x$ represents model parameters or weights lying in a convex domain $W\subset \RR^d$, $P_z$ is some unknown distribution over examples $z$ and $\ell(x,z)$ represents a loss function we will assume to be convex and smooth in $x$. Although we do not know $P_z$, we do know the loss $\ell$, and we have access to an i.i.d. dataset $Z=(z_1,\dots,z_T)$ that may have been obtained via user surveys or volunteer tests. Using this information, we would like to solve the optimization problem (\ref{eqn:stochopt}). The quality of a putative solution $\hat x$ is measured by the \emph{suboptimality} gap $\L(\hat x) - \L(x_\star)$ for $x_\star \in \argmin_{x\in W} \L(x)$.

In addition to solving (\ref{eqn:stochopt}), we also wish to \emph{preserve privacy} for the people who contributed to the dataset $Z$. To this end, we require our algorithms to be \emph{differentially private} \citep{dwork2006calibrating, dwork2014algorithmic}, which means that replacing any one $z_t$ with a different $z_t'$ has a negligible effect on $\hat x$. There is a delicate interplay between privacy, dataset size, and solution quality. As $T$ grows, the influence of any individual $z_t$ on $\hat x$ decreases, increasing privacy. However, for any finite $T$, one must necessarily leak \emph{some} information about the dataset in order to achieve a nontrivial $\L(\hat x)-\L(x_\star)$. The goal, then, is to minimize $\L(\hat x)-\L(x_\star)$ subject to the privacy constraint.

This problem has been well-studied in the literature, and by now the optimal tradeoffs are known and achievable\footnote{We focus on the harder stochastic optimization problem rather than empirical risk minimization, which is also well-studied, e.g. \citep{chaudhuri2011differentially, kifer2012private}.}. In particular, \cite{bassily2019private} exhibits a polytime algorithm that achieves $\L(\hat x)-\L(x_\star)\le \tilde O(1/\sqrt{T} + \sqrt{d}/\epsilon T)$ where $\epsilon$ is a measure of privacy loss (large $\epsilon$ means less private), and moreover they show that this bound is tight in the worst case. Then, \cite{feldman2020private} provides an improved algorithm that obtains the same guarantee in $O(T)$ time. 
% There has been significant further interest in this problem, considering 
Further work on this problem considers
assumptions on the geometry \citep{asi2021private}, gradient distributions \citep{kamath2021improved}, or smoothness \citep{bassily2020stability, kulkarni2021private}.

All private optimization algorithms we are aware of fall into one of two camps: either they employ some simple pre-processing that ``sanitizes'' the inputs to a non-private optimization algorithm (e.g. the empirically popular DP-SGD of \cite{abadi2016deep}), or they make a careful analysis of the dynamics of their algorithm (e.g. bounding the sensitivity of a single step of stochastic gradient descent). In the former case, the algorithm typically does \emph{not} achieve the optimal convergence rate for stochastic optimization. In the latter case, the algorithm becomes more rigidly tied to the privacy analysis, resulting in delicate ``theory-crafted'' methods that are less popular in practice.

In contrast, in the non-private setting, there is a simple and general technique to produce stochastic optimization algorithms with optimal convergence guarantees: the \emph{online-to-batch conversion} \citep{cesa2004generalization}. This method directly converts any \emph{online convex optimization} (OCO)\citep{shalev2011online, hazan2019introduction, orabona2019modern} algorithm into a stochastic optimization algorithm. OCO is a simple and elegant game-theoretic formulation of the optimization problem that has witnessed an explosion of diverse algorithms and techniques, so that this conversion result immediately implies a vast array of practical optimization algorithms.
In summary: producing \emph{private} stochastic optimization algorithms with optimal convergence rates is currently delicate and difficult, while producing non-private algorithms is essentially trivial.

Our goal is to rectify this issue. To do so, we produce a new \emph{differentially private} online-to-batch conversion. In direct analogy to the non-private conversion, our method converts any OCO algorithm into a \emph{private} stochastic optimization algorithm. After using our conversion, any OCO algorithm that achieves the optimal regret (the standard measure of algorithm quality in OCO), will automatically achieve the optimal suboptimality gap of $\tilde O(1/\sqrt{T} + \sqrt{d}/\epsilon T)$. Our conversion has additional desirable properties: although convexity is required for the guarantee on suboptimality, it is \emph{not} required for the privacy guarantee, meaning that the method can be easily applied to non-convex tasks (e.g. in deep learning). Further, by largely decoupling the privacy analysis from the algorithm design through this reduction, we can leverage the rich literature of OCO to build private algorithms with new \emph{adaptive} guarantees. For two explicit examples, (1) we construct an algorithm guaranteeing $\L(\hat  x)-\L(x_\star)\le \tilde O(\sigma /\sqrt{T}+ \sqrt{d}/\epsilon T)$, where $\sigma$ is the apriori unknown standard deviation in the gradient $\nabla \ell(w,z)$, and (2), we construct an algorithm guaranteeing $\L(\hat  x)-\L(x_\star)\le \tilde O(\|\hat x-x_1\|/\sqrt{T} + \sqrt{d}/\epsilon T) $, where $x_1$ is any user-supplied point. This last guarantee may have applications in private \emph{fine tuning} \citep{li2022private, yu2021differentially, hoory2021learning, kurakin2022toward, mehta2022large}, as the bound automatically improves when the algorithm is provided with a good initialization point.

\subsection{Preliminaries}

\paragraph{Problem setup} 

We define the loss function as $\ell: W \times Z \to \RR$ where $W \subseteq \RR^d$ is a convex domain and $Z=(z_1,\dots,z_T)$ is an element of $\mathcal{Z}^T$ for some dataspace $\mathcal{Z}$. We assume $Z$ is an i.i.d. dataset and $z_t\sim P_z$ for some unknown disribution over $\Z$. We then define $\L(x) = \E_{z\sim P_z}[\ell(x,z)]$.
% Given any online learning algorithm $\A$, we denote $w_t$ as the learned parameter of $\A$ in round $t$ before receiving $t$-th loss, and we construct $x_1,\ldots,x_T$ by $x_t = \sum_{i=1}^t \alpha_iw_i/\alpha_{1:t}$ such that $\alpha_1,\dots,\alpha_T$ is some positive sequence and $\alpha_{1:t} = \sum_{i=1}^t\alpha_i$. We also denote a sequence $\{X_1,\ldots,X_T\}$ of random variables (or random vectors) by $X_{1:T}$ for simplicity, and denote $\{X_{1:T}=x_{1:T}\}$ as the event $\{X_1=x_1,\ldots,X_T=x_T\}$.

Let $\|\cdot\|$ be a norm on $\RR^d$, with dual norm $\|\cdot\|_*$ defined by $\|g\|_*=\sup_{\|x\|\leq 1} \langle g,x\rangle$. By definition, $\langle g,x\rangle \leq \|g\|_*\|x\|$, (\emph{Fenchel-Young's inequality}).
We make the following assumptions:% on $\|\cdot\|$ and $\ell(x,z)$.

\begin{assumption}
\label{as:sc-norm}
$\|\cdot\|^2$ is $\lambda$-strongly convex w.r.t. $\|\cdot\|$: for all $x,y\in \RR^d$ and $g\in \partial\|x\|^2$,
\begin{align*}
    \|y\|^2 \geq \|x\|^2 + \langle g, y-x\rangle + \frac{\lambda}{2}\|y-x\|^2.
\end{align*}
\end{assumption}

\begin{assumption}
\label{as:bounded-domain}
$W$ has diameter at most $D$: $\forall$ $x,y\in W$, $\|x-y\|\leq D$.
\end{assumption}

\begin{assumption}
\label{as:Lipschitz}
$\ell(x,z)$ is $G$-Lipschitz in $x$: $\forall$ $x\in W, z\in Z$, $\|\nabla\ell(x,z)\|_* \leq G$.
\end{assumption}

\begin{assumption}
\label{as:smooth}
$\ell(x,z)$ is $H$-smooth in $x$: $\forall$ $x,y\in W, z\in Z$, $\|\nabla\ell(x,z) - \nabla\ell(y,z)\|_* \leq H\|x-y\|$.
\end{assumption}

\begin{assumption}
\label{as:sigma-G}
$\E[\|\nabla\L(x) - \nabla\ell(x,z)\|_*^2] \leq \sigma_G^2$ for all $x,z$.
\end{assumption}

\begin{assumption}
\label{as:sigma-H}
$\E[\|[\nabla\L(x)-\nabla\L(y)] - [\nabla\ell(x,z)-\nabla\ell(y,z)]\|_*^2] \leq \sigma_H^2\|x-y\|^2$ for all $x,y,z$.
\end{assumption}

% As a remark, \cite{shalev07online} proved that $\|\cdot\|_p^2$ is $2(p-1)$-strongly convex w.r.t. $\|\cdot\|_p$ for all $1<p\leq 2$.
% Moreover, 
Note that if $\ell$ is $G$-Lipschitz and $H$-smooth in $x$, then so it $\L$, and Assumption \ref{as:sigma-G} and \ref{as:sigma-H} hold with $\sigma_G\leq 2G$ and $\sigma_H\leq 2H$. Moreover, notice that Assumption \ref{as:bounded-domain} - \ref{as:sigma-H} depend on the norm $\|\cdot\|$. 
% For a different norm $\|\cdot\|'$, the constant may be different or the assumptions may no longer hold.
% Also, Assumption \ref{as:Lipschitz} implies Assumption \ref{as:sigma-G} with $\sigma_G \leq 2G$, and Assumption \ref{as:smooth} implies Assumption \ref{as:sigma-H} with $\sigma_H \leq 2H$.

% \begin{definition}
% \label{def:proper-norm}
% For a differentiable function $\ell:\RR^d\to\RR$ and a domain $W\subset \RR^d$, we say norm $\|\cdot\|$ on $\RR^d$ is a \textit{proper norm} w.r.t. $\ell, W$ (\highlight{or any name}) with parameter $(\lambda, D, G, H)$ if: (i) $\|\cdot\|^2$ is $\lambda$-strongly convex, (ii) the domain $W$ is bounded by diameter $D$, and (iii) $\ell$ is $G$-Lipschitz and $H$-smooth on $W$ (w.r.t. $\|\cdot\|$). Equivalently, for all $x,y\in W$,
% \[
% \begin{array}{ll}
%     1) & \|x-y\|^2 \geq \|y\|^2 + \langle g, y-x\rangle + \frac{\lambda}{2}\|y-x\|^2, \, \forall g\in\partial\|x\|^2; \\
%     2) & \|x-y\| \leq D; \\
%     3) & \|\nabla\ell(x)\|_* \leq G; \\
%     4) & \|\nabla\ell(x) - \nabla\ell(y)\|_* \leq H\|x-y\|.
% \end{array}
% \]
% \end{definition}
\paragraph{Differential Privacy} 

We now provide a formal definition of our privacy metric, differential privacy (DP). The definition hinges on the notion of \emph{neighboring datasets}: 
datasets $Z=(z_1,\ldots,z_T)$ and $Z'=(z_1',\ldots,z_T')$ in $\Z^T$ are said to be neighbors if $|Z-Z'|\triangleq|\{t\ |\  z_t\ne z'_t\}| = 1$. 
% Now we have the formal definition of differential privacy as follows:

\begin{definition}[$(\epsilon,\delta)$-DP \citep{dwork2014algorithmic}] A randomized algorithm $M: \Z^T \to \RR^d$ is $(\epsilon,\delta)$-differentially private for $\epsilon,\delta\geq0$ if for any neighboring $Z,Z'\in \mathcal{Z}^T$ and $S\in\RR^d$:
\begin{align*}
    \P\{M(Z) \in S\} &\le \exp(\epsilon)\P\{M(Z')\in S\} + \delta
\end{align*}
% When $\delta = 0$, the randomized algorithm $M$ is said to be $\epsilon$-differentially private ($\epsilon$-DP). 
\end{definition}

An alternative definition is \textit{R\'enyi differential privacy} (RDP), which is a generalization of DP that allows us to compose mechanisms more easily and achieve tighter privacy bounds in certain cases. 

\begin{definition}[$(\alpha,\epsilon)$-RDP \citep{mironov2017renyi}] A randomized mechanism $M: \Z^T \to \RR^d$ is said to be $(\alpha,\epsilon)$-R\'enyi differentially private for $\alpha>1,\epsilon\geq0$ if for any neighboring datasets $Z,Z' \in \mathcal{Z}^T$:
\begin{align*}
    D_{\alpha}(M(Z)||M(Z'))\le \epsilon
\end{align*}
where $D_{\alpha}(P||Q) \triangleq \frac{1}{\alpha -1}\log \E_{x \sim Q} \left(\frac{P(x)}{Q(x)}\right)$.
\end{definition}

RDP can be easily converted to the usual $(\epsilon,\delta)$-DP as follows \citep{mironov2017renyi}: if a randomized algorithm $M$ is $(\alpha,\epsilon)$-RDP, then it is also $(\epsilon+\frac{\log1/\delta}{\alpha-1},\delta)$-DP for all $\delta\in(0,1)$. In particular, if $M$ is $(\alpha, \alpha\rho^2/2)$-RDP for all $\alpha>1$, then it is also $(2\rho\sqrt{\log(1/\delta)}, \delta)$-DP for all $\delta\ge \exp(-\rho^2)$.

Throughout the paper, we also frequently use the notion of sensitivity:

\begin{definition}[Sensitivity] 
The sensitivity of a function $f: \Z^T \to \RR^d $ w.r.t. norm $\|\cdot\|$ is:
\begin{align*}
    \Delta_f = \sup_{|Z-Z'|=1}\|f(Z) - f(Z')\|_*.
\end{align*}
Almost all methods for ensuring differential privacy involve injecting noise whose scale increases with the sensitivity of the output. In other words, small sensitivity implies high privacy.
\end{definition}

% \paragraph{Notations}

% Throughout this paper, for scalars $\beta_1,\ldots,\beta_t$, we denote $\beta_{1:t} = \sum_{i=1}^t \beta_i$; for random variables $X_1,\ldots,X_t$, we denote $X_{1:t}=\{X_1,\ldots,X_t\}$.
\section{Diffentially Private Online-to-Batch}
\label{sec:OTB}

% $x_t-w_t$ converges of order $1/t$ => sensitivity decreases $1/t$

In this section, we present our main differentially private online-to-batch algorithm. To start, we need to define \emph{online convex optimization} (OCO). OCO is a ``game'' in which for $T$ rounds, $t=1,\ldots,T$, an algorithms predicts a parameter $w_t\in W$. It then receives a convex loss $\ell_t:W\to \RR$ and pays the loss $\ell_t(w_t)$. The quality of the algorithm is measured by the regret w.r.t. a competitor $u$, defined as $\reg_T(u)=\sum_{t=1}^T \ell_t(w_t)-\ell_t(u)$.

Online-to-batch algorithms \citep{cesa2004generalization} convert OCO algorithms (online learners) into stochastic convex optimization algorithms. In particular, for $\weight_1,\ldots,\weight_T>0$, the anytime online-to-batch conversion \citep{cutkosky2019anytime} defines the $t$-th loss as $\ell_t(w)=\langle \weight_t\nabla\ell(x_t,z_t),w\rangle$, where $x_t=\sum_{i=1}^t \frac{\weight_iw_i}{\weight_{1:t}}$.\footnote{Throughout this paper, we denote $\weight_{1:t}=\sum_{i=1}^t\weight_i$.} Convergence of anytime online-to-batch builds on the following key result, and its proof is in Appendix \ref{app:proof-A} for completeness.

% Online-to-batch conversion (\cite{cesa2004generalization}) convert convex optimization problem (i.e., solving $\min_{w\in W} \L(w)$ for convex $\L$) into online learning problem, which can be solved by many well-studied algorithms such as Online Mirror Descent \cite{}, Follow-The-Regularized-Leader \cite{}, etc. Recently, \cite{cutkosky2019anytime} provides anytime online-to-batch algorithm that guarantees so-called last-iterate convergence, i.e., the last-iterate error $\L(x_T)-\L(x^*)$ converges to $0$ as $T\to\infty$. More specifically, one of its essential theorems relates the last-iterate error to the regret:

\begin{restatable}[\cite{cutkosky2019anytime}]{theorem}{basiclemma}
\label{lem:basic-lemma}
For any sequence of $\weight_t>0, g_t\in\RR^d$, suppose an online learner predicts $w_t$ and receives $t$-th loss $\ell_t(w)=\langle g_t,w\rangle$. Define $x_t=\sum_{i=1}^t \frac{\weight_iw_i}{\weight_{1:t}}$ where $\weight_{1:t}=\sum_{i=1}^t\weight_i$. Then for any convex and differentiable $\L$,
\begin{align*}
    \weight_{1:T}(\L(x_T)-\L(u)) \leq \reg_T(u) + \sum_{t=1}^T \langle \weight_t\nabla\L(x_t) - g_t, w_t-u\rangle, \ \forall u\in \RR^d.
\end{align*}
\end{restatable}
% The proof is provided in Appendix \ref{app:proof-A} for completeness. 

A tighter bound is possible \citep{joulani2020simpler}, but the simpler expression above suffices for our analysis.
As an immediate result, choosing $\weight_t=1$ and $g_t=\weight_t\nabla\ell(x_t,z_t)$ satisfies $\E[g_t|x_t]=\weight_t\nabla\L(x_t)$, so $\E[\L(x_T)-\L(u)]\leq \E[\reg_T(u)]/T$. Therefore, any online learner with sub-linear regret guarantees convergence for the last iterate $x_T$. Notice that due to the formulation of the anytime online-to-batch, the iterate $x_t$ is \emph{stable}. Consider the case where $\weight_t = 1$ for all t.
% The original motivation for this type of conversion is to obtain an "anytime" or \emph{last iterate} guarantee. However, the more important property of the conversion is that the iterates $x_t$ produced by this conversion are \emph{extremely stable}: consider the simple case $\weight_t=1$.
Then, we have $\|x_t-x_{t-1}\|= \|w_t-x_{t-1}\|/t\le D/t$ \emph{regardless of what the online learner does}. This guarantee is significantly stronger than the classical online-to-batch result of \cite{cesa2004generalization}. We would like to take advantage of this stability to design our private online-to-batch algorithm. Intuitively, the algorithm has much lower sensitivity due to the stability of the iterates, which can be exploited to improve privacy.

% Although this result is satisfying in most scenarios, one limitation is the $g_t$ has to be an unbiased estimator of $\weight_t\nabla\L(x_t)$. If $g_t$ is biased, the inner product term no longer goes to zero in expectation, and the analysis becomes more complicated due to this additional error.

% \highlight{explain why simply using $\nabla\ell_t(w_t)+r_t$ is not a good idea.}
% We may think that we can simply add a mean-zero noise to $g_t$ to make the algorithm differentially-private. However, ...

Our goal will be to make the entire sequence $g_1,\dots,g_T$ private, which, in turn, makes the entire algorithm private. To do so, we would like to add noise to the $g_t$'s while still having $g_t$ be a good estimate of $\nabla \L(x_t)$. 
In standard non-private online-to-batch, $g_t$ is usually defined as the stochastic gradient $\nabla\ell(x_t,z_t)$. However, directly adding noise to $\nabla\ell(x_t,z_t)$ is not a good idea because its sensitivity is $O(1)$ (more specifically the sensitivity is bounded by $2G$ by Lipschitzness). Consequently, we need to add noise to $g_t$ whose variance is of order $O(1/\epsilon)$ in order to achieve $\epsilon$-differential privacy.

Instead, we express $g_t=\sum_{i=1}^t \delta_i$, where $\delta_i=\weight_i\nabla\ell(x_i,z_i)-\weight_{i-1}\nabla\ell(x_{i-1},z_i)$ and $\weight_0\equiv 0$. Since we assume $\ell$ is smooth, if we set $\weight_i=1$ then $\|\delta_i\|_*\leq H\|x_i-x_{i-1}\|\leq DH/i$, i.e., the sensitivity of $\delta_i$ is $O(1/i)$.
% By definition of $x_i$, $x_i-x_{i-1}=(w_i-x_{i-1})/i$, so the sensitivity of $\delta_i$ is $O(1/i)$. 
As a result, we can privately estimate $g_t$ with error roughly $\tilde O(1/t\epsilon)$ using the \emph{tree aggregation mechanism} \citep{dwork2010differential, chan2011private}, an advanced technique that privately estimates running sums, such as our $\sum_{i=1}^t \delta_i$. Compared to directly adding noise to $\nabla\ell(x_t,z_t)$, this method adds less noise ($\tilde O(1/t\epsilon)$ compared to $O(1/\epsilon)$) and thus allows us to achieve the optimal rate.

On the other hand, after using these advanced aggregation techniques, $g_t$ is no longer an conditionally unbiased estimator of $\weight_t\nabla\L(x_t)$. More specifically, although it remains the case that $\E_{z_i}[\delta_i|z_1,\ldots,z_{i-1}]=\weight_i\nabla\L(x_i)-\weight_{i-1}\nabla\L(x_{i-1})$, it is not necessarily true that $\E[g_t|z_1,\ldots,z_{t-1}]\neq \weight_t\nabla\L(x_t)$. Unbiasedness plays a key role in standard convergence analyses, but we will need a much more delicate analysis. 

Moreover, although we mostly discuss the $\weight_t =1$ case above for intuition, our algorithm is analyzed using the general case $\weight_t = t^k$ for $k\ge 1$. The guiding principle for this formula is the sensitivity of $\delta_t$. For $k\geq 1$, the sensitivity of $\delta_t$ is of order $O(t^{k-1})$. For $k=1$, this is a constant sensitivity, which is particularly intuitive for analysis. For $k=0$ (i.e. the standard weighting in online-to-batch), the sensitivity is actually $O(1/t)$, which is much more complicated to analyze. In order to apply the tree aggregation easily, we want the sensitivity of $\delta_t$ to be polynomial in $t$, rather than the inverse polynomial $1/t$, so we define $\beta_t=t^k$ and ask $t\geq k$.
Furthermore, in all cases except for the parameter-free case (Section \ref{sec:parameter-free}), our results hold for all $k\geq 1$. In the parameter-free case, we choose $k=3$ for algebraic reasons.

% Formally, 
The pseudo-code is presented in Algorithm \ref{alg:DP-OTB}, which has linear time complexity. It is similar to anytime online-to-batch, while we replace $g_t$ with the more complicated definition and add noise $\gamma_t$ generated by the $\textsc{Noise}$ subroutine, which implements the tree aggregation. More specifically, given random noises $\{R_t\}$, $\textsc{Noise}(t)$ returns $\sum_{i\in I_t}R_i$, where $I_t$ is the set of cumulative sums of the binary expansion of $t$. That is, for some $n\ge \lfloor \log_2(t)\rfloor+1$, we define $\mathrm{bin}(t)\in\{0,1\}^n$ by $t=\sum_{i=1}^n \mathrm{bin}(t)[i] 2^{n-i}$. Then, $I_t$ consists of all non-zero sums of the form $\sum_{k=1}^i \mathrm{bin}(t)[i] 2^{n-i}$. For examples, $7=4+2+1$, so $I_7=\{4,6,7\}$; $8=8$, so $I_8=\{8\}$.
% In each round $t$, we update $x_t$ with $w_t$ from the online learner and compute gradient difference $\delta_t$. We then compute $g_t=g_{t-1}+\delta_t$, the estimator of $\weight_t\nabla\L(x_t)$, and add noise $\gamma_t$ generated by the tree mechanism. Finally, we give the noised linear loss $\ell_t(w)=\langle g_t+\gamma_t,w\rangle$ to the online learner as the $t$-th loss.

\begin{algorithm}[]
\begin{algorithmic}[1]
    \State \textbf{Input:} OCO algorithm $\A$ with domain $W$, positive sequence $\{\weight_t\}$, 
    % and sequence of \textit{independent} and \textit{mean-zero} random vector $\{R_t\}$.
    distribution $\D$, dataset $Z$.
    \State \textbf{Initialize:} Set $\weight_0 = 0, x_0 = 0$ and $g_0 = 0$. Set global variable $\mathcal{R}=\{\}$.
    % Set random matrix $R$ such that $R_{ij}\sim \N(0,\sigma^2)$.
    \For{$t=1,\ldots,T$}
        \State Get $w_t$ from $\A$ and compute $x_t = (\weight_{1:t-1}x_{t-1}+\weight_tw_t)/\weight_{1:t}$.
        \State Compute gradient difference $\delta_t = \weight_t\nabla\ell(x_t,z_t)-\weight_{t-1}\nabla\ell(x_{t-1},z_t)$.
        \State Update $g_t = g_{t-1}+\delta_t$ and generate noise $\gamma_t=\textsc{Noise}(t)$.
        \State Send $\ell_t(w) = \langle g_t+\gamma_t, w\rangle$ to $\A$ as the $t$-th loss.
    \EndFor
    
    \vspace{0.5em}
    
    \Function{Noise}{$t$}
        \State \textbf{Initialize: } Set $k=0$, $I_t=\{\}$, $\mathrm{bin}(t)$ be the binary encoding of $t$, and $n=\lfloor \log_2 t\rfloor+1$.
        \For{$i=1,\ldots,n$}
            \State If $\mathrm{bin}(t)[i] = 1$, update $k = k+2^{n-i}$ and $I_t=I_t\cup\{k\}$.
        \EndFor 
        \State Generate noise $\tilde R_t\sim \D$, compute $R_t=\sigma_t\tilde R_t$, and update $\mathcal{R}=\mathcal{R}\cup\{R_t\}$.
        % \State Compute $\sigma_t^2$ and $R_t=\sigma_t \tilde R_t$, and update $\mathcal{R}=\mathcal{R}\cup\{R_t\}$.
        \State \textbf{Return} $\gamma_t = \sum_{i\in I_t} R_i$.
    \EndFunction
\end{algorithmic}
\caption{Differentially-Private Online-to-Batch}
\label{alg:DP-OTB}
\end{algorithm}

\paragraph{Convergence}

Following Theorem \ref{lem:basic-lemma} and Fenchel-Young's inequality, Algorithm \ref{alg:DP-OTB} satisfies:
\begin{align}
    \label{eq:basic-decomposition}
    \weight_{1:T}(\L(x_T)-\L(x^*)) 
    &\leq \reg_T(x^*) + \sum_{t=1}^T \langle \weight_t\nabla\L(x_t) - g_t - \gamma_t, w_t-x^*\rangle \\
    % \intertext{By Fenchel-Young's inequality and the assumption that $\|w_t-x^*\|\leq D$:}
    &\leq \reg_T(x^*) + \sum_{t=1}^TD\|\weight_t\nabla\L(x_t)-g_t\|_* + \sum_{t=1}^T D\|\gamma_t\|_*.
    \label{eq:basic-decomposition-2}
\end{align}

\eqref{eq:basic-decomposition-2} decomposes the suboptimality gap $\L(x_T)-\L(x^*)$ into three components: (i) regret $\reg_T(x^*)$, (ii) error associated with variance of $g_t$, measured by $\|\weight_t\nabla\L(x_t)-g_t\|_*$, and (iii) error from DP mechanism, measured by $\|\gamma_t\|_*$.
To get a tight bound, we observe that $\weight_t\nabla\L(x_t)-g_t$ and $\gamma_t$ are sums of conditionally mean-zero random vectors (Lemma \ref{lem:variance-expectation-bound} in Appendix \ref{app:proof-A}). With a martingale bound in high dimension with general norm (Lemma \ref{lem:martingale-bound-general-norm}), we can derive the following bounds. In particular, we show (Lemma \ref{lem:variance-expectation-bound}) that if Assumption \ref{as:sc-norm} - \ref{as:sigma-H} hold and set $\weight_t=t^k$, then
\begin{align}
    \E[\|\weight_t\nabla\L(x_t)-g_t\|_*^2] \leq 4(k+1)^2(\sigma_G^2+D^2\sigma_H^2)t^{2k-1}/\lambda, \label{eq:variance-norm}.
\end{align}
Moreover, we show (Lemma \ref{lem:privacy-expectation}) that if $\E[R_t]=0$ and $\E[\|R_t\|_*^2]\leq \bar \sigma_t^2$, then
\begin{align}
    \E[\|\gamma_t\|_*^2] \leq 2(\max_{i\leq t}\bar \sigma_i^2)\log_2(2t)/\lambda. \label{eq:privacy-norm}
\end{align}
% Note that the error from variance \eqref{eq:variance-norm} does not depend on the DP mechanism, while \eqref{eq:privacy-norm} depends on the variance $\sigma_t^2$. Therefore, the next question is: what is the appropriate noise we should add to make the algorithm differential private and converging at the same time?

\paragraph{Privacy}

The next step is to determine how much noise is sufficient for Algorithm \ref{alg:DP-OTB} to be RDP. We first make the following assumption on the distribution $\D$.

% For a function $f:Z\mapsto f(Z)$ where $Z$ denotes a dataset, we can make $f$ differentially-private by adding noise to it. Typically, to achieve $\epsilon$-differential privacy, the variance of the noise is of order $\Delta^2/\epsilon^2$, where $\Delta$ is the sensitivity of $f$ measured in associated norm. For example, Laplace mechanism \cite{...} measures sensitivity in 1-norm and adds a Laplace noise, and Gaussian mechanism \cite{mironov2017renyi} measures in 2-norm and adds a Gaussian noise.

% In this paper, we want to be as general as possible, so we consider DP mechanisms on general norms. Formally, we make the following assumption:

% \highlight{slightly modified the definition and changed notation to $V$ for variance}

\begin{definition}[$(V,\alpha)$-RDP distribution]
\label{def:RDP-dist}
A distribution $\D$ on $\RR^d$ is said to be a \emph{$(V,\alpha)$-RDP distribution on $\|\cdot\|$} if it satisfies that for $R\sim\D$ (i) $\E[R]=0$, (ii) $\E[\|R\|_*^2]\leq V$, and (iii) for all $\rho>0$ and $\mu,\mu'\in\RR^d$, if $\sigma^2\geq \|\mu-\mu'\|_*^2/\rho^2$, then $D_\alpha(\sigma R+\mu\|\sigma R+\mu')\leq \alpha\rho^2/2$.\footnote{Here we slightly abuse the notation. For random vectors $X,Y$, $D_\alpha(X\|Y)$ denotes the R\'enyi divergence of the underlying distributions of $X$ and $Y$.}
\end{definition}

\begin{remark}
\label{rmk:Gaussian-RDP-dist}
% Multivariate Gaussian distribution satisfies this definition w.r.t. 2-norm. 
A standard $(V,\alpha)$-RDP distribution on the $L_2$ norm is the multivariate normal distribution. Let $R\sim\N(0,I)$ then it is clear that $\E[R]=0$ and $\E[\|R\|_2^2]=d$. For the third condition, we can show that for all $\mu,\mu'$ and $\alpha>1$, $D_\alpha(\N(\mu, \sigma^2I)\|\N(\mu',\sigma^2I))\leq \alpha\|\mu-\mu'\|_2^2/2\sigma^2$, which is further bounded by $\alpha\rho^2/2$ for all $\sigma^2 \geq \|\mu-\mu'\|_2^2/\rho^2$  (Lemma \ref{lemma:renyigaussian}) .
Therefore, $\N(0,I)$ is a $(d,\alpha)$-RDP distribution for all $\alpha>1$.
\end{remark}

% \highlight{For the Gaussian mechanism, the variance should be $\sigma^2\geq K\|\mu-\mu'\|_*^2/\epsilon$ instead of $\epsilon^2$ (replace $\sigma^2$ in $D_\alpha(\N(0,\sigma^2 I\| \N(\mu,\sigma^2 I) \leq \alpha\|\mu\|^2/2\sigma^2)$), should we change it to match Gaussian mechanism? (and that should reduce $\log^2 T$ to $\log T$ dependence in $\sigma_t^2$ \eqref{eq:RDP-noise}.}

As its name suggests, adding noise sampled from RDP distribution is sufficient to make a deterministic algorithm RDP.
% Consider any algorithm $f:\Z\mapsto \RR^d$, and define its sensitivity measured in $\|\cdot\|$ as
% \begin{align*}
%     \Delta = \sup_{|Z-Z'|=1} \|f(Z)-f(Z')\|_*.
% \end{align*}
Consider function $\hat{f}(Z)=f(Z)+\sigma R$, where $R\sim \D$ and $\D$ is a $(V,\alpha)$-RDP distribution on $\|\cdot\|$. Let $\Delta$ be the sensitivity of $f$ w.r.t $\|\cdot\|$. Set $\sigma^2\geq \Delta^2/\rho^2$, then by Definition \ref{def:RDP-dist},
\begin{align*}
    D_\alpha(\hat{f}(Z)\|\hat{f}(Z')) 
    &= D_\alpha(f(Z)+\sigma R \| f(Z')+\sigma R) \leq \alpha\rho^2/2.
\end{align*}
Thus, $\hat{f}$ is $(\alpha,\alpha\rho^2/2)$-RDP.
Moreover, we can compose RDP mechanisms via the tree aggregation mechanism. Specifically, we set $\weight_t=t^k$ and define the variance $\sigma_t^2$ in Algorithm \ref{alg:DP-OTB} as follows:
\begin{align}
    \sigma_t^2 
    &= \frac{4(k+1)^2}{\rho^2}(G+H\max_{i\in[t]}\|w_i-x_{i-1}\|)^2\log_2(2T)t^{2k-2}, \label{eq:RDP-noise}
\end{align}
% Since we assume domain $W$ is bounded by diameter $D$, $\|w_i-x_{i-1}\|\leq D$. 
We assume $\|w_i-x_{i-1}\|\leq D$. 
Upon substituting \eqref{eq:RDP-noise} into \eqref{eq:privacy-norm}, we get $\E[\|R_t\|_*^2] \leq \bar \sigma_t^2 \leq V\sigma_t^2$ and:
\begin{align}
    \E[\|\gamma_t\|_*^2]
    % &\leq \frac{2K\log_2(2t)}{\lambda}\left(\frac{2(k+1)(G+DH)t^{k-1}\log_2(2T)}{\rho}\right)^2. \\
    &\leq \frac{8(k+1)^2V(G+DH)^2}{\lambda\rho^2} \log_2^2(2T)t^{2k-2}.
    \label{eq:gamma-bound}
\end{align}
The following theorem shows that Algorithm \ref{alg:DP-OTB} is R\'enyi differentially private if we define $\sigma_t^2$ as in \eqref{eq:RDP-noise}. Its proof is presented in Appendix \ref{app:tree-aggregation-RDP}. Note that the privacy guarantee does \emph{not} require i.i.d. $Z$.
% \begin{restatable}{theorem}{TreeRDP}
% \label{thm:main-privacy}
% Suppose $\|\cdot\|^2$ is $\lambda$-strongly convex, $W$ is bounded by $D$, and $\ell$ is $G$-Lipschitz and $H$-smooth. Also assume $\D$ is a $(V,\alpha)$-RDP distribution. If we set $\weight_t=t^k$ and set $\sigma_t^2$ as: 
% \begin{align}
%     \sigma_t^2 
%     &= \frac{4(k+1)^2}{\rho^2}(G+H\max_{i\in[t]}\|w_i-x_{i-1}\|)^2\log_2(2T)t^{2(k-1)}, \label{eq:RDP-noise}
% \end{align}
% then Algorithm \ref{alg:DP-OTB} is $(\alpha,\alpha\rho^2/2)$-DP.
% \end{restatable}
\begin{restatable}{theorem}{TreeRDP}
\label{thm:main-privacy}
Suppose $\|\cdot\|^2$ is $\lambda$-strongly convex, $W$ is bounded by $D$, $\ell$ is $G$-Lipschitz and $H$-smooth, and $\D$ is a $(V,\alpha)$-RDP distribution. If $\weight_t = t^k$ and $\sigma_t^2$ is as defined in \eqref{eq:RDP-noise}, then Algorithm \ref{alg:DP-OTB} is $(\alpha,\alpha\rho^2/2)$-DP for all datasets $Z$.
\end{restatable}

\paragraph{Main Result.} Now we can combine all the previous results to prove the privacy and convergence guarantee of our algorithm.

\begin{theorem}
\label{thm:DP-OTB-general-norm}
Suppose Assumption \ref{as:sc-norm} - \ref{as:sigma-H} hold, and $\D$ is a $(V,\alpha)$-RDP distribution. If we set $\weight_t=t^k$ and define $\sigma_t^2$ as in \eqref{eq:RDP-noise}, then Algorithm \ref{alg:DP-OTB} is $(\alpha,\alpha\rho^2/2)$-RDP and $\E[\L(x_T) - \L(x^*)]$ is bounded by:
\begin{align*}
    % \E[\L(x_T) - \L(x^*)] \leq & 
    \frac{(k+1)\E[\reg_T(x^*)]}{T^{k+1}}
    + \frac{2(k+1)^2D}{\sqrt{\lambda}}\left(\frac{\sigma_G+D\sigma_H}{\sqrt{T}}+\frac{\sqrt{2V}(G+DH)\log_2(2T)}{\rho T}\right).
\end{align*}
Moreover, recall that the online learner receives $t$-th loss $\ell_t(w)=\langle g_t+\gamma_t,w\rangle$. It holds that
\begin{align*}
    \E[\|g_t+\gamma_t\|_*^2] \leq 3t^{2k}\left(G^2 + \frac{4(k+1)^2}{\lambda}\left(\frac{(\sigma_G+D\sigma_H)^2}{t}+\frac{2V(G+DH)^2\log_2^2(2T)}{\rho^2t^2}\right) \right).
\end{align*}
\end{theorem}

\begin{remark}
\label{rmk:base-result}
As an example, let's consider the Gaussian distribution $\N(0,I)$ on the 2-norm, which is a $(d,\alpha)$-RDP distribution for all $\alpha>1$ (Remark \ref{rmk:Gaussian-RDP-dist}). For many popular online learners (OSD, OMD, FTRL), if $\E[\|\nabla\ell_t(w_t)\|_*^2]\leq \hat G^2$ for all $t$, then $\E[\reg_T(x^*)] \leq O(D\hat{G}\sqrt{T})$. Hence, Theorem \ref{thm:DP-OTB-general-norm} with $\D=\N(0,I)$ implies that
\begin{align*}
    \E[\L(x_T)-\L(x^*)] = O\left(\frac{D(G+D\sigma_H)}{\sqrt{T}} + \frac{\sqrt{d}D(G+DH)\log T}{\rho T}\right).
\end{align*}
This bound is of $\tilde O(1/\sqrt{T}+\sqrt{d}/\rho T)$, which can be translated  to an equivalent $(\epsilon,\delta)-$DP bound of $\tilde O(1/\sqrt{T} + \sqrt{d\log(1/\delta)}/\epsilon T)$. This bound matches the optimal rate for private stochastic optimization with convex and smooth losses (\cite{bassily2019private}).
\end{remark}

\begin{proof}[Proof of Theorem \ref{thm:DP-OTB-general-norm}]
From Eq. \eqref{eq:basic-decomposition-2}, we have:
\begin{align*}
    & \E[\L(x_T)-\L(x^*)] 
    \leq \frac{1}{\weight_{1:T}} \E\left[ \reg_T(x^*) + D\sum_{t=1}^T (\|\weight_t\nabla\L(x_t)-g_t\|_* + \|\gamma_t\|_*) \right] \\
    \intertext{Recall the bounds of $\E[\|\weight_t\nabla\L(x_t)-g_t\|_*^2]$ and $\E[\|\gamma_t\|_*^2]$ in \eqref{eq:variance-norm} and \eqref{eq:gamma-bound}. By Jensen's inequality, $\E[\|X\|_*]\leq \sqrt{\E[\|X\|_*^2]}$. Moreover, since $\weight_t=t^k$ and $k\geq 1$, it holds that $\weight_{1:T} \geq T^{k+1}/k+1$, so:}
    &\leq \frac{\E[\reg]}{\weight_{1:T}}  + \frac{D}{\weight_{1:T}}\sum_{t=1}^T \frac{2(k+1)(\sigma_G+D\sigma_H) t^{k-\frac{1}{2}}}{\sqrt{\lambda}} + \frac{\sqrt{8V}(k+1)(G+DH)\log_2(2T) t^{k-1}}{\sqrt{\lambda}\rho}\\
    % &\quad \ + \sum_{t=1}^T \frac{\sqrt{8V}(k+1)D(G+DH)\log_2(2T) t^{k-1}}{\weight_{1:T}\sqrt{\lambda}\rho} \\ 
    % \intertext{Since we define $\weight_t=t^k$ and $k\geq 1$, it holds that $\weight_{1:T} \geq T^{k+1}/k+1$, so:}
    % &\leq \frac{(k+1)\E[\reg_T(x^*)]}{T^{k+1}} + \frac{2(k+1)^2D(\sigma_G+D\sigma_H)}{\sqrt{\lambda T}} \\
    % &\quad \ + \frac{\sqrt{8V} (k+1)^2 D(G+DH)\log_2(2T)}{\sqrt{\lambda}\rho T} \\
    &\leq \frac{(k+1)\E[\reg_T(x^*)]}{T^{k+1}} + \frac{2(k+1)^2D}{\sqrt{\lambda}}\left(\frac{\sigma_G+D\sigma_H}{\sqrt{T}}+\frac{\sqrt{2V}(G+DH)\log_2(2T)}{\rho T}\right).
\end{align*}
For the second part of the theorem, 
\begin{align*}
    &\E[\|g_t+\gamma_t\|_*^2]
    \leq 3\E[\|\weight_t\nabla\L(x_t)\|_*^2 + \|g_t-\weight_t\nabla\L(x_t)\|_*^2 + \|\gamma_t\|_*^2] \\
    \intertext{We bound the first term by Lipschitzness, the second by \eqref{eq:variance-norm}, and the third by \eqref{eq:gamma-bound}}
    &\leq 3t^{2k}G^2 + \frac{12(k+1)^2(\sigma_G^2+D^2\sigma_H^2)t^{2k-1}}{\lambda}
    + \frac{24V(k+1)^2(G+DH)^2\log_2^2(2T)t^{2k-2}}{\lambda \rho^2} \\
    &= 3t^{2k}\left(G^2 + \frac{4(k+1)^2}{\lambda}\left(\frac{(\sigma_G+D\sigma_H)^2}{t}+\frac{2V(G+DH)^2\log_2^2(2T)}{\rho^2t^2}\right) \right).
\end{align*}
\end{proof}

\section{The Optimistic Case}
\label{sec:optimistic}

In this section, we show that choosing an optimistic online learner \citep{chiang2012online, rakhlin2013online, steinhardt2014eg} will accelerate our DP online-to-batch algorithm. Optimistic algorithms  are provided with additional ``hints'' in the form of $\hat \ell_t(w)=\langle \hat g_t, w\rangle$ as an approximation of the true loss $\ell_t(w)=\langle g_t,w\rangle$, and they can incorporate $\hat \ell_t$ to decide $w_t$. The regret of an optimistic algorithm depends on the quality of hints: if $\hat{g}_t\approx g_t$, then it achieves a low regret.  Formally, in this paper, we say an online learning algorithm is \textit{optimistic w.r.t. norm $\|\cdot\|$} if its regret is the following:
\begin{equation}
    \reg_T(x^*) \leq O\left( D\sqrt{\sum_{t=1}^T \| \hat{g}_t - g_t\|_*^2 } \right), \label{eq:optimistic-OL}
\end{equation}
where $D$ denotes the diameter of the learner's domain.

A common choice of the hint $\hat{g}_t$ is $g_{t-1}$, the gradient in the last round since intuitively, one could expect $g_{t-1}\approx g_t$ when the loss functions are smooth. In this section, we also follow this choice. Recall that in Algorithm \ref{alg:DP-OTB}, the online learner receives $t$-th loss $\ell_t(w) = \langle g_t+\gamma_t, w\rangle$, where $g_t=\sum_{i=1}^t\delta_i$ is the sum of gradient differences, and $\gamma_t$ is some noise. Therefore, we define the $t$-th hint as $\hat{g}_t = g_{t-1}+\gamma_{t-1}$.

\begin{restatable}{theorem}{ThmOptimistic}
\label{thm:optimistic}
% Suppose $\|\cdot\|^2$ is $\lambda$-strongly convex, $W$ is bounded by $D$, and $\ell$ is $G$-Lipschitz and $H$-smooth. 
Suppose Assumption \ref{as:sc-norm} - \ref{as:smooth} hold, and $\D$ is a $(V,\alpha)$-RDP distribution. Set $\weight_t=t^k$ and $\sigma_t^2$ as defined in \eqref{eq:RDP-noise},
If the online learner is optimistic (satisfying \eqref{eq:optimistic-OL}) with $t$-th gradient $\bar{g}_t=g_t+\gamma_t$ and $t$-th hint $\hat{g}_t=\bar{g}_{t-1}$, then
\begin{align*}
    \frac{\E[\reg_T(x^*)]}{\weight_{1:T}} 
    % \leq \sqrt{3}(k+1)^2D(G+DH)\left(1+\frac{4\sqrt{V}\log_2(2T)}{\sqrt{\lambda}\rho}\right)\frac{1}{T^{3/2}}.
    \leq O\left( D(G+DH)\left(1+\frac{\sqrt{V}\log T}{\sqrt{\lambda}\rho}\right)\frac{1}{T^{3/2}} \right).
\end{align*}
% Consequently, if we further assume Assumption \ref{as:sigma-G} and \ref{as:sigma-H}, then
% \begin{align*}
%     \E[\L(x_T) - \L(x^*)]
%     &\leq \frac{2(k+1)^2D}{\sqrt{\lambda}}\left(\frac{\sigma_G+D\sigma_H}{\sqrt{T}}+\frac{2\sqrt{K}(G+DH)\log_2^{3/2}T}{\rho T}\right).
% \end{align*}
\end{restatable}

The proof is in Appendix \ref{app:optimistic}.
As an immediate corollary, if we further assume Assumption \ref{as:sigma-G} and \ref{as:sigma-H}, then Theorem \ref{thm:DP-OTB-general-norm} applies. Together with this theorem, they imply that optimistic learners achieve the following convergence rate that is adaptive to the variance:
\begin{align*}
    \E[\L(x_T)-\L(x^*)] = O\left( \frac{D(\sigma_G+D\sigma_H)}{\sqrt{T}} + \frac{\sqrt{d}D(G+DH)\log T}{\rho T} \right).
\end{align*}
Compared to the bound in the non-optimistic case (Remark \ref{rmk:base-result}), this bound has $\sigma_G$ instead of $G$ in the first term. Thus, when the gradient $\nabla\ell(x_t,z_t)$ has low variance, i.e., $\sigma_G\ll G$, the optimistic bound outperforms the standard bound.
\section{The Strongly Convex Case}
\label{sec:strongly-convex}

In this section, we prove that in the case of strong convexity, our algorithm can be improved by regularizing the loss of online learner.
If $\L$ is strongly convex, then we can prove a similar result to Theorem \ref{lem:basic-lemma} (the proof is in Appendix \ref{app:strongly-convex}).

\begin{restatable}{lemma}{LemmaSCBasic}
\label{lem:strongly-convex-regularized-loss}
Suppose $\L$ is $\mu$-strongly convex w.r.t. $\|\cdot\|$. If we replace $\ell_t(w)=\langle g_t+\gamma_t,w\rangle$ with $\bar \ell_t(w)=\ell_t(w)+\frac{\weight_t\mu}{4}\|w-x_t\|^2$ in Algorithm \ref{alg:DP-OTB}, and denote the associated regret by $\overline \reg_T$, then
\begin{align*}
    \weight_{1:T} (\L(x_T) - \L(x^*))
    &\leq \overline\reg _T(x^*) +\sum_{t=1}^T \langle \weight_t\nabla\L(x_t)-g_t-\gamma_t, w_t-x^*\rangle - \frac{\weight_t\mu}{8}\|w_t-x^*\|^2 \\
    &\leq \overline\reg_T(x^*) + \sum_{t=1}^T \frac{2\|\weight_t\nabla\L(x_t)-g_t-\gamma_t\|_*^2}{\weight_t \mu}.
\end{align*}
\end{restatable}

% First, recall that $x_t$ are determined by $\{g_1+\gamma_1,\ldots,g_{t-1}+\gamma_{t-1}\}$, and we proved in Theorem \ref{thm:main-privacy}

Compared to Lemma \ref{lem:basic-lemma} and Equation \ref{eq:basic-decomposition-2}, in the strongly convex case, there is an additional term $-\frac{\weight_t\mu}{8}\|w_t-x^*\|^2$, which allows the improved convergence rate:

\begin{theorem}
\label{thm:strongly-convex}
Suppose Assumption \ref{as:sc-norm} - \ref{as:sigma-H} hold, and $\D$ is a $(V,\alpha)$-RDP distribution. Also suppose $\L$ is $\mu$-strongly convex. Set $\weight_t=t^k$ and $\sigma_t^2$ as defined in \eqref{eq:RDP-noise}. Then $\E[\L(x_T) - \L(x^*)]$ is bounded by:
\begin{align*}
    % \E[\L(x_T)-\L(x^*)] &\leq 
    \frac{(k+1)\E[\overline\reg_T(x^*)]}{T^{k+1}} 
    + \frac{16(k+1)^3}{\lambda \mu} \left( \frac{(\sigma_G+D\sigma_H)^2}{T} + \frac{2V(G+DH)^2\log_2^2(2T)}{\rho^2T^2} \right).
\end{align*}
Moreover, for all $\bar g_t\in \partial \bar \ell_t(w_t)$,
\begin{align*}
    \E[\|\bar g_t\|_*^2] \leq t^{2k}\left( 4G^2 + \mu^2D^2 + \frac{16(k+1)^2}{\lambda}\left(\frac{(\sigma_G+D\sigma_H)^2}{t} + \frac{2V(G+DH)^2\log_2^2(2T)}{\rho^2t^2}\right) \right).
\end{align*}
\end{theorem}

\begin{remark}
As an example, again consider the case $\D=\N(0,I)$ with $L_2$ norm. Online subgradient descent (OSD) with appropriate learning rate on $\mu_t$-strongly convex losses $\ell_t$ achieves:
% the following regret:
\begin{align*}
    \reg_T(u) \leq \sum_{t=1}^T \frac{\|g_t\|_2^2}{2\sum_{i=1}^t\mu_i},
\end{align*}
where $g_t\in \partial\ell_t(w_t)$. In our case, $\|\cdot\|_2^2$ is $2$-strongly convex and the regularized loss $\bar \ell_t$ is $\frac{\weight_t\mu}{2}$-strongly convex, i.e. $\mu_t=\mu t^k/2$ and $\sum_{i=1}^t \mu_i = O(\mu t^{k+1})$. Therefore, by Theorem \ref{thm:strongly-convex},
\begin{align*}
    \E[\overline\reg_T(x^*)] 
    % &\lesssim \sum_{t=1}^T \frac{t^{k-1}}{\mu}\left(G^2+\mu^2D^2+ \frac{(\sigma_G+D\sigma_H)^2}{t}+\frac{\alpha(G+DH)^2\log_2^{3/2}T}{\rho^2t^2}\right) \\
    % &\leq \frac{T^k}{\mu}\left(G^2 + \mu^2 D^2 + \frac{(\sigma_G+D\sigma_H)^2}{T}+\frac{\alpha(G+DH)^2\log_2^{3/2}T}{\rho^2T^2} \right).
    &\leq O\left( \frac{T^k}{\mu}\left( G^2+\mu^2D^2 + \frac{(\sigma_G+D\sigma_H)^2}{T} + \frac{d(G+DH)^2\log^2T}{\rho^2T^2}  \right) \right).
\end{align*}
Consequently,
\begin{align*}
    \E[\L(x_T)-\L(x^*)]
    \leq O\left(  \frac{(G+\mu D+D\sigma_H)^2}{\mu T} + \frac{d(G+DH)^2\log^2T}{\mu\rho^2T^2} \right).
\end{align*}
This again matches the optimal private convergence rates.
\end{remark}

\begin{proof}[Proof of Theorem \ref{thm:strongly-convex}]
We have already bounded $\E[\|\weight_t\nabla\L(x_t)-g_t\|_*^2]$ and $\E[\|\gamma_t\|_*^2]$ in \eqref{eq:variance-norm} and \eqref{eq:gamma-bound} respectively, so
\begin{align*}
    \E[\|\weight_t\nabla\L(x_t)-g_t-\gamma_t\|_*^2]
    &\leq 2\E[\|\weight_t\nabla\L(x_t)-g_t\|_*^2+\|\gamma_t\|_*^2] \\
    % &\leq 8(k+1)^2(\sigma_G^2+D^2\sigma_H^2)t^{2k-1}/\lambda, \label{eq:variance-norm} + \frac{4K\log_2(2t)}{\lambda}\left(\frac{2(k+1)(G+DH)t^{k-1}\log_2(2T)}{\rho}\right)^2.
    &\leq \frac{8(k+1)^2t^{2k-1}}{\lambda}\left((\sigma_G+D\sigma_H)^2+ \frac{2V(G+DH)^2\log_2^2(2T)}{\rho^2 t}\right).
\end{align*}
Upon substituting this into Lemma \ref{lem:strongly-convex-regularized-loss} and replace $\weight_t=t^k$, we get:
\begin{align*}
    & \weight_{1:T} \E[\L(x_T)-\L(x^*)] \\
    \leq & \E[\overline\reg_T(x^*)] + \sum_{t=1}^T \frac{16(k+1)^2t^{k-1}}{\lambda\mu}\left((\sigma_G+D\sigma_H)^2+ \frac{2V(G+DH)^2\log_2^2(2T)}{\rho^2 t}\right) \\
    % \intertext{Recall $\weight_t=t^k$ and $\weight_{1:T}\geq T^{k+1}/(k+1)$.}
    \leq & \E[\overline\reg_T(x^*)] + \frac{16(k+1)^2}{\lambda \mu} \left( (\sigma_G+D\sigma_H)^2T^k + \frac{2V(G+DH)^2\log_2^2(2T)T^{k-1}}{\rho^2} \right).
\end{align*}
Dividing both sides by $\weight_{1:T} \geq T^{k+1}/(k+1)$ proves the first part of the theorem.

For the second part, recall that $\bar \ell_t(w)=\langle g_t+\gamma_t,w\rangle + \frac{\weight_t\mu}{4}\|w-x_t\|^2$. Therefore, for all $\bar g_t\in\partial \bar\ell_t(w)$, $\bar g_t = g_t+\gamma_t + \frac{\beta_t\mu}{4}v$, where $v\in \partial\|w_t-x_t\|^2$.
% Since the linear part is differentiable,
% \begin{align*}
%     \partial\bar{\ell}_t(w) = g_t+\gamma_t + \frac{\weight_t\mu}{4}\partial\|w-x_t\|^2. 
% \end{align*}
% By chain rule of sub-differentials (Corollary 16.72 \cite{bauschke2011convex}), let $\phi(r)=r^2$ and $f(w)=\|w-x_t\|$, then $\|w-x_t\|^2 = \phi\circ f(w)$ and
% \begin{align*}
%     \partial \|w-x_t\|^2 
%     &= \{\alpha u : \alpha\in \partial \phi(f(w)), u\in\partial f(w) \} \\
%     &= \{2\|w-x_t\| u : u\in\partial \|w-x_t\|\}.
% \end{align*}
% By assumption, $\|w-x_t\|\leq D$. By triangular inequality, $\|x\|-\|y\| \leq \|x-y\|$, and thus $\|\cdot\|$ is $1$-Lipschitz. Equivalently, $\|u\|_*\leq 1$ for all $u\in\partial \|w-x_t\|$. Therefore, for all $v\in \partial\|w-x_t\|^2$, $\|v\|_*\leq 2D$.
We follow the same argument in Theorem \ref{thm:DP-OTB-general-norm}: 
\begin{align*}
    \E[\|\bar g_t\|_*^2]
    &\leq 4\E[\|\weight_t\nabla\L(x_t)\|_*^2 + \|\weight_t\nabla\L(x_t)-g_t\|_*^2 + \|\gamma_t\|_*^2 + \|\tfrac{\weight_t\mu}{4} v\|_*^2] \\
    \intertext{Here $v\in\partial\|w_t-x_t\|^2$. We bound the first term by Lipschitz, the second by \eqref{eq:variance-norm}, and the third by \eqref{eq:gamma-bound}. Moreover, by chain rule (Proposition \ref{prop:subdiff-chain-rule}), we can show that $\|v\|_* \leq 2D$, so:}
    &\leq 4t^{2k}G^2 + \frac{16(k+1)^2(\sigma_G+D\sigma_H)^2t^{2k-1}}{\lambda} \\
    &\quad \ + \frac{32V(k+1)^2(G+DH)^2t^{2k-2}\log_2^2(2T)}{\lambda\rho^2} + \mu^2D^2t^{2k} \\
    &\leq t^{2k}\left( 4G^2 + \mu^2D^2 + \frac{16(k+1)^2}{\lambda}\left(\frac{(\sigma_G+D\sigma_H)^2}{t} + \frac{2V(G+DH)^2\log_2^2(2T)}{\rho^2t^2}\right) \right).
\end{align*}
\end{proof}

\section{Parameter-free Algorithm}
\label{sec:parameter-free}

In this section, we apply Algorithm~\ref{alg:DP-OTB} with a ``parameter-free'' online learner. These are algorithms that guarantee $\reg_T(u)\le \tilde O(\|u\|\sqrt{T})$ for all competitors $u$ simultaneously \citep{orabona2016coin, cutkosky2018black, mhammedi2020lipschitz}. By shifting coordinates, it is possible to obtain $\reg_T(u)\le \tilde O(\|u-x_0\|\sqrt{T})$ for any pre-specified point $x_0$. Thus, if $x_0\approx u$ is some good initialization, perhaps generated by pretraining, this bound yields significantly smaller regret than if we had used the worst-case diameter bound $\|u-x_0\|\le D$.

In order to obtain this refined bound with  privacy, we need to make a small modification to our conversion. For simplicity, we focus on Euclidean space with 2-norm, and we assume the distribution $\D$ is in addition sub-Gaussian, i.e., for $R\sim \D$,
\begin{align*}
    \P\{\sup_{\|a\|_2=1} \langle R, a\rangle \geq \epsilon \} \leq \exp\left(-\frac{\epsilon^2}{2\sigma^2}\right).
\end{align*}
In general, the proof extends to any Banach space and any distribution $\D$ that concentrates on it.

In previous analysis (Equation \eqref{eq:basic-decomposition-2}), we roughly bound $\|w_t-x^*\|\leq D$. However, in this section, we come up with a finer high probability bound that maintains a dependence on $\|w_t\|, \|x^*\|$. We then replace the loss $\ell_t(w)$ in Algorithm \ref{alg:DP-OTB} with a regularized loss $\ell_t(w)+\xi_t\|w\|_2+\nu_t\|w\|_2^2$, and we show that the new algorithm with regularized loss can achieve a parameter-free bound. The complete proof is presented in Appendix \ref{app:parameter-free}.

\begin{restatable}{theorem}{ThmParameterFreeA}
\label{thm:parameter-free-2-norm-step-1}
Suppose w.r.t. 2-norm, $W$ is bounded by $D$ and $\ell$ is $G$-Lipschitz and $H$-smooth. Suppose $\D$ is $(V,\alpha)$-RDP distribution and is $\sigma_\D$-sub-Gaussian. If we set $\weight_t=t^3$ (i.e. $k=3$) and set $\sigma_t^2$ as defined in \eqref{eq:RDP-noise}, then with probability at least $1-\delta$,
\begin{align*}
    \L(x_T)-\L(x^*) \leq \frac{4}{T^4}\left( \reg_T(x^*) + \sum_{t=1}^T \xi_t(\|w_t\|_2+\|x^*\|_2) + \nu_t(\|w_t\|_2^2+\|x^*\|_2^2) \right).
\end{align*}
where $C$ is a universal constant, $A=8\sqrt{2}C^2, A'=8\sqrt{d}\sigma_\D C^2, \kappa=1+DH/G$, 
% $\Phi=\sqrt{\log(16dT\log(2\kappa T)/\delta)}$, 
and
\begin{align*}
    &\xi_t = AG\Phi t^{5/2} + A'(G+DH)\frac{\Phi\log_2(2T)t^2}{\rho} , \
    \nu_t = 28AH\Phi t^{5/2}, \ 
    \Phi = \sqrt{\log\frac{20dT\log(2\kappa T)}{\delta}}.
\end{align*}
\end{restatable}

\begin{restatable}{theorem}{ThmParameterFree}
\label{thm:parameter-free-main}
Following the assumptions and notations in Theorem \ref{thm:parameter-free-2-norm-step-1}, if we replace $\ell_t(w)$ in Algorithm \ref{alg:DP-OTB} with regularized loss $\bar \ell_t(w) = \ell_t(w) + \xi_t\|w\|_2 + \nu_t\|w\|_2^2$ and denote the associated regret as $\overline \reg_t$, then with probability at least $1-\delta$, $\L(x_T)-\L(x^*)$ is bounded by:
\begin{align*}
    % \L(x_T)-\L(x^*) 
    % &\leq 
    \frac{4\overline \reg_T(x^*)}{T^4} + \frac{8A\|x^*\|(G+28\|x^*\|H)\Phi}{\sqrt{T}} + \frac{8A'\|x^*\|(G+DH)\Phi\log_2(2T)}{\rho T}.
\end{align*}
Moreover, with probability at least $1-\delta$, for all $t$ and for all $w\in W, \bar g_t \in \partial \bar \ell_t(w)$,
\begin{align*}
    \|\bar g_t\|_2 \leq Gt^3 + A(2G+57DH)\Phi t^{5/2} + 2A'(G+DH)\frac{\Phi\log_2(2T)t^2}{\rho}.
\end{align*}
\end{restatable}

\begin{remark}
If $\|\bar g_t\|_2 \leq \hat G$ for all $t$, parameter-free algorithms achieve regret bound $\reg_T(u) = \tilde O(\|u\|_2\hat G\sqrt{T})$. Therefore, with $\D=\N(0,I)$, $\D$ is $1$-sub-Gaussian and $A'=O(\sqrt{d})$, so Theorem \ref{thm:parameter-free-main} implies that with probability $1-\delta$,
\begin{align*}
    \frac{\overline\reg_T(x^*)}{T^4} = \tilde O\left( \frac{\|x^*\|_2G}{\sqrt{T}} + \frac{\|x^*\|_2(G+DH)}{T} + \frac{\|x^*\|_2\sqrt{d}(G+DH)}{\rho T^{3/2}} \right).
\end{align*}
Consequently,
\begin{align*}
    \L(x_T) - \L(x^*) \leq \tilde O\left( \frac{\|x^*\|_2(G+DH)}{\sqrt{T}} + \frac{\|x^*\|_2\sqrt{d}(G+DH)}{\rho T} \right).
\end{align*}
\end{remark}

\begin{proof}[Proof of Theorem \ref{thm:parameter-free-main}]
The regularized regret satisfies
\begin{align*}
    \overline \reg_T(x^*) 
    % &= \sum_{t=1}^T [\ell_t(w_t) + \xi_t\|w_t\|_2 + \nu_t\|w_t\|_2^2] - [\ell_t(x^*)+\xi_t\|x^*\|_2^2 + \nu_t\|x^*\|_2^2] \\
    &= \reg_T(x^*) + \sum_{t=1}^T \xi_t(\|w_t\|_2-\|x^*\|_2)+\nu_t(\|w_t\|_2^2-\|x^*\|_2^2).
\end{align*}
Hence, upon substituting this equation into Theorem \ref{thm:parameter-free-2-norm-step-1}, we get: with probability at least $1-\delta$,
\begin{align*}
    &\L(x_T)-\L(x^*) 
    \leq \frac{4\overline\reg_T(x^*)}{T^4} + \frac{8}{T^4} \sum_{t=1}^T \xi_t\|x^*\|_2 + \nu_t\|x^*\|_2^2 \\
    &\leq \frac{4\reg_T(x^*)}{T^4} + \frac{8(AG\Phi\|x^*\|_2+28AH\Phi\|x^*\|_2^2)}{\sqrt{T}} + \frac{8A'(G+DH)\Phi\log_2(2T)\|x^*\|}{\rho T}.
    % &\leq \frac{4\reg_T(x^*)}{T^4} + \frac{8A\|x^*\|(G+28\|x^*\|H)\Phi}{\sqrt{T}} + \frac{8A'\|x^*\|(G+DH)\Phi\log_2(2T)}{\rho T}.
\end{align*}
The second inequality is from $\sum_{t=1}^T \xi_t \leq T\xi_T$ because $\xi_t$ is increasing with $t$ (so is $\nu_t$). 

For the second part of the theorem, for each fixed $t$ and for all $\bar g_t\in \bar\ell_t(w_t)$, $\bar g_t = g_t+\gamma_t + \xi_t u+ 2\nu_t w_t$, where $u\in \partial \|w_t\|_2$ and thus $\|u\|_2 \leq 1$. Therefore,
\begin{align*}
    \|\bar g_t\|_2
    &\leq \| g_t - \weight_t\nabla\L(x_t)\|_2 + \|\weight_t\nabla\L(x_t)\|_2 + \|\gamma_t\|_2 + \|\xi_t u\|_2 + \|2\nu_tw \|_2
\end{align*}
By Lipschitzness, $\|\weight_t\nabla\L(x_t)\|_2 \leq Gt^3$. Since $W$ is bounded, $\|2\nu_tw_t\|_2\leq 2D\nu_t$. Also, $\|\xi_tu\|_2\leq \xi_t$. Moreover, we can prove (in Eq. \ref{eq:variance-concentration} and \ref{eq:privacy-concentration}) that for each $t$, with probability at least $1-\delta/2T$,
\begin{align*}
    & \|\weight_t\nabla\L(x_t) - g_t\|_2
    \leq 8C^2\Phi \sqrt{\sum_{i=1}^t i^4(G+H\|w_i-x_{i-1}\|_2)^2} 
    \leq A\Phi (G+DH) t^{5/2}, \\
    & \|\gamma_t\|_2 \leq \frac{A'}{\rho}(G+DH)\Phi \log_2(2T)t^2.
\end{align*}
Upon taking the union bound for all $t$ and  the definition of $\xi_t,\nu_t$, we get the desired bound.
% For the third term, we can prove (in Equation \eqref{eq:privacy-concentration}) that for each $t$, with probability at least $1-\delta/2T$,
% \begin{align*}
%     \|\gamma_t\|_2 \leq \frac{A'}{\rho}(G+DH)\Phi \log_2(2T)t^2.
% \end{align*}
% In conclusion, upon taking the union bound over all $t$ and substituting the definition of $\xi_t,\nu_t$, we get: with probability at least $1-\delta$, for all $t$,
% \begin{align*}
%     \| \bar g_t \|_2 
%     &\leq A\Phi(G+DH)t^{5/2} + Gt^3 + A'\Phi(G+DH) \frac{\log_2(2T)t^2}{\rho} \\
%     &\quad \ + \left(A\Phi Gt^{5/2} +  A'\Phi(G+DH)\frac{\log_2(2T)t^2}{\rho}\right) + 56A\Phi DHt^{5/2} \\
%     &\leq Gt^3 + A(2G+57DH)\Phi t^{5/2} + 2A'(G+DH)\frac{\Phi\log_2(2T)t^2}{\rho}.
% \end{align*}
\end{proof}
\section{Conclusion}\label{sec:conclusion}
We have presented a new online-to-batch conversion that produces private stochastic optimization algorithms on smooth losses. Online algorithms achieving the optimal $O(\sqrt{T})$ regret automatically achieve the optimal $\tilde O(1/\sqrt{T} + \sqrt{d}/\epsilon T)$ convergence rate. Combining this technique with the literature on online learning can yield new private optimization algorithms. 

\textbf{Limitations:} Our algorithm requires smoothness, and unlike some other bounds, we cannot tolerate large $H$. In the worst case when $H=\sqrt{T}$ and $\sigma_H=H$, our standard bound in Remark \ref{rmk:base-result} becomes $O(1)$. In other words, we need to assume $H=o(\sqrt{T})$ to ensure a non-trivial bound. Removing this restriction would significantly improve the generality of the procedure,

The dependence on $H$ comes from the sensitivity of $\delta_t$ (Lemma \ref{lem:variance-expectation-bound}), where we apply smoothness to bound $\|\weight_t(\nabla\ell(x_t,z_t)-\nabla\ell(x_{t-1},z_t))\|\leq \beta_t H \|x_t-x_{t-1}\|$ and use the stability of $x_t$ to further bound $\|x_t-x_{t-1}\| \leq D\beta_t/\beta_{1:t}$, which are necessary steps in order to bound the sensitivity of $\delta_t$ by $O(t^{k-1})$. Hence, it's not clear how to remove the smoothness assumption.

\bibliographystyle{apalike}
\bibliography{references.bib}

%%%%%%%%%%%%%%%%%%%%%%%%%%%%%%%%%%%%%%%%%%%%%%%%%%%%%%%%%%%%
\section*{Checklist}

%%% BEGIN INSTRUCTIONS %%%
% The checklist follows the references.  Please
% read the checklist guidelines carefully for information on how to answer these
% questions.  For each question, change the default \answerTODO{} to \answerYes{},
% \answerNo{}, or \answerNA{}.  You are strongly encouraged to include a {\bf
% justification to your answer}, either by referencing the appropriate section of
% your paper or providing a brief inline description.  For example:
% \begin{itemize}
%   \item Did you include the license to the code and datasets? \answerYes{See Section~\ref{gen_inst}.}
%   \item Did you include the license to the code and datasets? \answerNo{The code and the data are proprietary.}
%   \item Did you include the license to the code and datasets? \answerNA{}
% \end{itemize}
% Please do not modify the questions and only use the provided macros for your
% answers.  Note that the Checklist section does not count towards the page
% limit.  In your paper, please delete this instructions block and only keep the
% Checklist section heading above along with the questions/answers below.
%%% END INSTRUCTIONS %%%

\begin{enumerate}

\item For all authors...
\begin{enumerate}
  \item Do the main claims made in the abstract and introduction accurately reflect the paper's contributions and scope?
    \answerYes{}
  \item Did you describe the limitations of your work?
    \answerYes{See discussion in Section \ref{sec:conclusion}.}
  \item Did you discuss any potential negative societal impacts of your work?
    \answerNo{This paper addresses mathematical problems, and we do not anticipate negative societal impact.}
  \item Have you read the ethics review guidelines and ensured that your paper conforms to them?
    \answerYes{}
\end{enumerate}

\item If you are including theoretical results...
\begin{enumerate}
  \item Did you state the full set of assumptions of all theoretical results?
    \answerYes{See Section \ref{sec:intro} for all assumptions.}
        \item Did you include complete proofs of all theoretical results?
    \answerYes{In the main text, we only leave the proofs for the most important results because of page limitation. However, the complete proofs of all results are included in the supplementary material.}
\end{enumerate}

\item If you ran experiments...
\begin{enumerate}
  \item Did you include the code, data, and instructions needed to reproduce the main experimental results (either in the supplemental material or as a URL)?
    \answerNo{We did not run experiments.}
  \item Did you specify all the training details (e.g., data splits, hyperparameters, how they were chosen)?
    \answerNo{We did not run experiments.}
        \item Did you report error bars (e.g., with respect to the random seed after running experiments multiple times)?
    \answerNo{We did not run experiments.}
        \item Did you include the total amount of compute and the type of resources used (e.g., type of GPUs, internal cluster, or cloud provider)?
    \answerNo{We did not run experiments.}
\end{enumerate}

\item If you are using existing assets (e.g., code, data, models) or curating/releasing new assets...
\begin{enumerate}
  \item If your work uses existing assets, did you cite the creators?
    \answerNo{We did not run experiments.}
  \item Did you mention the license of the assets?
    \answerNo{We did not run experiments.}
  \item Did you include any new assets either in the supplemental material or as a URL?
    \answerNo{We did not run experiments.}
  \item Did you discuss whether and how consent was obtained from people whose data you're using/curating?
    \answerNo{We did not run experiments.}
  \item Did you discuss whether the data you are using/curating contains personally identifiable information or offensive content?
    \answerNo{We did not run experiments.}
\end{enumerate}

\item If you used crowdsourcing or conducted research with human subjects...
\begin{enumerate}
  \item Did you include the full text of instructions given to participants and screenshots, if applicable?
    \answerNo{We did not use crowdsourcing nor conduct research with human subjects.}
  \item Did you describe any potential participant risks, with links to Institutional Review Board (IRB) approvals, if applicable?
    \answerNo{We did not use crowdsourcing nor conduct research with human subjects.}
  \item Did you include the estimated hourly wage paid to participants and the total amount spent on participant compensation?
    \answerNo{We did not use crowdsourcing nor conduct research with human subjects.}
\end{enumerate}

\end{enumerate}

%%%%%%%%%%%%%%%%%%%%%%%%%%%%%%%%%%%%%%%%%%%%%%%%%%%%%%%%%%%%

\newpage

\appendix

\section{Proofs for Convergence (Section \ref{sec:OTB})}
\label{app:proof-A}

\basiclemma*

\begin{proof}
Since $\L$ is convex,
\begin{align*}
    \sum_{t=1}^T \weight_t ( \L(x_t) - \L(u) )
    &\leq \sum_{t=1}^T \weight_t \langle \nabla\L(x_t), x_t-w_t+w_t-u \rangle \\
    &= \sum_{t=1}^T \langle\nabla\L(x_t), \weight_t(x_t-w_t)\rangle + \langle \weight_t\nabla\L(x_t)-g_t+g_t, w_t-u\rangle.
\end{align*}
By construction of $x_t$, it holds that $\weight_t(x_t-w_t) = \weight_{1:t-1}(x_{t-1}-x_t)$. By convexity,
\begin{align*}
    \langle \nabla\L(x_t), \weight_{1:t-1}(x_{t-1}-x_t) \rangle \leq \weight_{1:t-1} \left( \L(x_{t-1}) - \L(x_t) \right).
\end{align*}
Next, we move $\sum \weight_t\L(x_t)$ to the right, giving:
\begin{align*}
    -\weight_{1:T} \L(u) 
    &\leq \sum_{t=1}^T \left( \weight_{1:t-1}\L(x_{t-1}) - \weight_{1:t}\L(x_t) + \langle \weight_t\nabla\L(x_t)-g_t+g_t,w_t-u\rangle \right) \\
    % &= -\weight_{1:T}\nabla\L(x_T) + \sum_{t=1}^T \langle\weight_t\nabla\L(x_t)-g_t+g_t, w_t-x^*\rangle \\
    &= -\weight_{1:T}\L(x_T) + \reg_T(u) + \sum_{t=1}^T \langle \weight_t\nabla\L(x_T) - g_t,w_t-u\rangle,
\end{align*}
where the equality follows from (i) the telescopic sum of $\weight_{1:t-1}\nabla\L(x_{t-1})+\weight_{1:t}\nabla\L(x_t)$ and (ii) the definition of regret that $\reg_T(u)=\sum_{t=1}^T \langle g_t,w_t-u\rangle$.
\end{proof}

\begin{restatable}{lemma}{GeneralMartingaleBound}
\label{lem:martingale-bound-general-norm}
Suppose $\|\cdot\|^2$ is $\lambda$-strongly convex w.r.t. $\|\cdot\|$, and let $\{X_t\}$ be a sequence of random vectors such that (i) $\E[\|X_t\|_*]<\infty$ and (ii) $\E[X_{t+1}\mid X_{1:t}]=0$ for all $t$. Then,
\begin{align*}
    \E\left[ \left\| \sum_{t=1}^T X_t \right\|_*^2 \right]
    &\leq \frac{2}{\lambda}\sum_{t=1}^T\E[\|X_t\|_*^2].
\end{align*}
% Consequently, by Jensen's inequality,
% \begin{align*}
%     \E\left[ \left\| \sum_{t=1}^T X_t \right\|_* \right]
%     &\leq \sqrt{\frac{2}{\lambda}\sum_{t=1}^T\E[\|X_t\|_*^2]}.
% \end{align*}
\end{restatable}

\begin{proof}
We use the regret approach to prove this statement. For simplicity, denote $M_T=\sum_{t=1}^T X_t$. Consider an online learner which receives $\ell_t(x)=\langle X_t, x\rangle$ as $t$-th loss and updates $w_{t+1}$. Then by definition of regret, for any $u$,
\begin{align*}
    - \langle M_T, u\rangle \leq \reg_T(u) - \sum_{t=1}^T \langle X_t, w_t\rangle.
\end{align*}
Since $w_t$ only depends on $X_{1:t-1}$ but \textit{not} on $X_t$, $w_t$ is constant given $X_{1:t-1}$. Therefore,
\begin{align*}
    \E[\langle X_t, w_t\rangle]
    &= \E_{X_{1:t-1}}\E_{X_t}[\langle X_t, w_t\rangle | X_{1:t-1}] \\
    &= \E_{X_{1:t-1}} \left[\langle \E_{X_t}[X_t|X_{1:t-1}], w_t\rangle \right] = 0,
\end{align*}
where the second equality follows from the assumption that $\E[X_t|X_{1:t-1}]=0$. Therefore,
\begin{align*}
    \E[\langle M_T, -u\rangle] \leq \E[\reg_T(u)].
\end{align*}
Recall the definition of the dual norm that $\|M_T\|_* = \sup_{\|x\|=1} \langle M_T,x\rangle$. Therefore, if we define $u^*=\|M_T\|_*\argmax_{\|u\|=1}\langle M_T,-u\rangle$, then it holds that
\begin{align*}
    \langle M_T,-u^*\rangle = \|M_T\|_*^2 \quad\textit{and}\quad \|u^*\|=\|M_T\|_*.
\end{align*}
Let the follow-the-regularized-leader (FTRL) algorithm be the online learner. \cite{orabona2019modern} proved that for any regularizer $\psi$ that is $\lambda$-strongly convex w.r.t. $\|\cdot\|$,  FTRL achieves the following regret:
\begin{align*}
    \reg_T(u) \leq \frac{\psi(u)}{\eta} +\frac{\eta}{2\lambda} \sum_{t=1}^T\|X_t\|_*^2.
\end{align*}
Since we assume $\|\cdot\|^2$ is $\lambda$-strongly convex w.r.t. $\|\cdot\|$, we can define $\psi(x)=\|x\|^2$ and get:
\begin{align*}
    \E[\langle M_T, -u^*\rangle]
    = \E[\|M_T\|_*^2] 
    &\leq \E[\reg_T(u^*)] 
    \leq \E\left[ \frac{\|M_T\|_*^2}{\eta} + \frac{\eta}{2\lambda} \sum_{t=1}^T \|X_t\|_*^2 \right].
\end{align*}
Equivalently, upon moving terms around we have:
\begin{align*}
    \E[\|M_T\|_*^2]
    &\leq \E\left[ \frac{\eta^2}{2\lambda (\eta-1)} \sum_{t=1}^T\|X_t\|_*^2 \right] 
    \leq \E\left[\frac{2}{\lambda}\sum_{t=1}^T\|X_t\|_*^2\right].
\end{align*}
The second inequality holds because $\inf_\eta \frac{\eta^2}{\eta-1}=4$ when $\eta=2$.
\end{proof}

In this paper, we will always set $\weight_t=t^k$ for some $k\geq 1$. The following proposition gives relevant bounds for $\weight_t-\weight_{t-1}$ and $\weight_t^2/\weight_{1:t}$.

\begin{proposition}
\label{prop:alphas}
If $\weight_t=t^k$, then (i) $\weight_t-\weight_{t-1} \leq kt^{k-1}$ and (ii) $\weight_t^2/\weight_{1:t} \leq (k+1)t^{k-1}$.
% \begin{align*}
%     (1) \quad & \weight_t-\weight_{t-1} \leq kt^{k-1}, \\
%     (2) \quad & \weight_t^2/\weight_{1:t} \leq (k+1)t^{k-1}.
% \end{align*}
\end{proposition}

\begin{proof}
For the first part, by mean value theorem, there exists some $\tau\in [t-1,t]$ such that
\begin{align*}
    \weight_t - \weight_{t-1} = t^k - (t-1)^k
    = k\tau^{k-1} \leq kt^{k-1}.
\end{align*}
For the second part, for any increasing function $f$, it holds that $\sum_{i=1}^t f(i) \geq \int_0^t f(x) \, dx$, so
\begin{align*}
    \weight_{1:t} = \sum_{i=1}^t i^k \geq \int_0^t x^k \, dx = \frac{t^{k+1}}{k+1}.
\end{align*}
Hence, $\weight_t^2/\weight_{1:t} \leq (k+1)t^{k-1}$.
\end{proof}

\begin{restatable}{lemma}{LemmaVarianceBound}
\label{lem:variance-expectation-bound}
Suppose $\|\cdot\|^2$ is $\lambda$-strongly convex, $W$ is bounded by $D$, and $\ell$ is $G$-Lipschitz and $H$-smooth. If we set $\weight_t=t^k$, then
\[
% \|\delta_t\|_*\leq |\weight_t-\weight_{i-1}|G + (\weight_i^2/\weight_{1:i})H\|w_t-x_{t-1}\|.
\|\delta_t\|_* \leq (k+1) (G+H\|w_t-x_{t-1}\|) t^{k-1}.
\]
If we further assume Assumption \ref{as:sigma-G} and \ref{as:sigma-H},
% i.e., for all $x,y,z$, $\E[\|\nabla\ell(x,z)-\nabla\L(x)\|_*^2]\leq \sigma_G^2$ and $\E[\|[\nabla\ell(x,z)-\nabla\ell(y,z)]-[\nabla\L(x)-\nabla\L(y)]\|_*^2]\leq \sigma_H^2\|x-y\|^2$, 
% \begin{align*}
%     & \E[\|\nabla\ell(x,z)-\nabla\L(x)\|_*^2]\leq \sigma_G^2, \\
%     & \E[\|[\nabla\ell(x,z)-\nabla\ell(y,z)]-[\nabla\L(x)-\nabla\L(y)]\|_*^2]\leq \sigma_H^2\|x-y\|^2,
% \end{align*}
then $\weight_t\nabla\L(x_t)-g_t = \sum_{i=1}^tX_i$ such that:
\begin{align*}
(i)\ & X_i = [\weight_i\nabla\L(x_i)-\weight_{i-1}\nabla\L(x_{i-1})] - [\weight_i\nabla\ell(x_i,z_i)-\weight_{i-1}\nabla\ell(x_{i-1},z_i)], \\
(ii)\ & \E[X_i|z_{1:i-1}]=0, \quad \text{and} \quad
(iii)\ \E[\|X_i\|_*^2] 
% \leq 2(\weight_i-\weight_{i-1})^2\sigma_G^2 + 2(\weight_i^2/\weight_{1:i})^2 D^2\sigma_H^2.
\leq 2(k+1)^2(\sigma_G^2+D^2\sigma_H^2)i^{2k-2}.
\end{align*}
Consequently,
\[
\E[\|\weight_t\nabla\L(x_t)-g_t\|_*^2] \leq 4(k+1)^2(\sigma_G^2+D^2\sigma_H^2)t^{2k-1}/\lambda.
\]
\end{restatable}

\begin{proof}
For the first part, note that
\begin{align*}
    \delta_t &= \weight_t\nabla\ell(x_t,z_t)-\weight_{t-1}\nabla\ell(x_{t-1},z_t) \\
    &= (\weight_{t}-\weight_{t-1}) \nabla\ell(x_{t-1},z_t) + \weight_t(\nabla\ell(x_t,z_t)-\nabla\ell(x_{t-1},z_t)).
\end{align*}
Since $\ell$ is $G$-Lipschitz and $H$-smooth, $\|\nabla\ell(x_{t-1},z_t)\|_*\leq G$ and
\begin{align*}
    \|\nabla\ell(x_t,z_t)-\nabla\ell(x_{t-1},z_t)\|_* 
    &\leq H\|x_t-x_{t-1}\| \leq (\weight_t/\weight_{1:t})H\|w_t-x_{t-1}\|,
\end{align*}
where the second inequality follows from the definition of $x_t$ that $\weight_{1:t}(x_t-x_{t-1})=\weight_t(w_t-x_{t-1})$. The first result then follows from Proposition \ref{prop:alphas}.

For the second part, by telescopic sum $\weight_t\nabla\L(x_t) = \sum_{i=1}^t \weight_i\nabla\L(x_i) - \weight_{i-1}\nabla\L(x_{i-1})$, and recall that $g_t=\sum_{i=1}^t \delta_i$. Therefore,
\begin{align*}
    \weight_t\nabla\L(x_t)-g_t 
    &= \sum_{i=1}^t [\weight_i\nabla\L(x_i) - \weight_{i-1}\nabla\L(x_{i-1})] - [\weight_i\nabla\ell(x_i,z_i)-\weight_{i-1}\nabla\ell(x_{i-1},z_i)].
\end{align*}
We denote each summand by $X_i$, and we can check $X_i$ satisfies condition (ii) and (iii). First, since we assume $\nabla\L(w) = \E_z[\nabla\ell(w,z)]$ for all $w$, it holds that $\E[X_i|z_{1:i-1}]=0$. Second, we decompose $X_i$ in the same way as the first part: 
\begin{align*}
    \E[\|X_i\|_*^2]
    &= \E[ \|(\weight_i-\weight_{i-1})[\nabla(\L(x_{i-1})-\nabla\ell(x_{i-1},z_i)] \\
    &\quad\ + \weight_{i}( [\nabla\L(x_i)-\nabla\ell(x_i,z_i)] - [\nabla\L(x_{i-1})-\nabla\ell(x_{i-1},z_i)] ) \|_*^2 ] \\
    \intertext{Recall the assumption of $\sigma_G^2$ and $\sigma_H^2$ and that $\|x_i-x_{i-1}\| \leq \weight_i/\weight_{1:i}\|w_i-x_{i-1}\|$.}
    &\leq 2(\weight_i-\weight_{i-1})^2\sigma_G^2 + 2\weight_i^2 \sigma_H^2\|x_i-x_{i-1}\|^2 \\
    &\leq 2(\weight_i-\weight_{i-1})^2\sigma_G^2 + 2(\weight_i^2/\weight_{1:i})^2 \sigma_H^2\|w_i-x_{i-1}\|^2 \\
    &\leq 2(k+1)^2(\sigma_G^2 + D^2\sigma_H^2)i^{2k-2}.
\end{align*}
The last inequality follows from the assumption that $\|w_i-x_{i-1}\|\leq D$ and Proposition \ref{prop:alphas}.

The last part of the theorem is a direct result from Lemma \ref{lem:martingale-bound-general-norm}:
\begin{align*}
    \E[ \| \nabla\L(x_t)-g_t \|_*^2]
    \leq \frac{2}{\lambda} \sum_{i=1}^t \E[\|X_i\|_*^2] 
    % &\leq \frac{2}{\lambda} \sum_{i=1}^t 2k^2i^{2k-2}\sigma_G^2 + 2(k+1)^2i^{2k-2}D^2\sigma_H^2 \\
    % &\leq \frac{2}{\lambda} \sum_{i=1}^t 2(k+1)^2(\sigma_G^2+D^2\sigma_H^2) i^{2k-2} \\
    % &\leq 4(k^2\sigma_G^2+(k+1)^2D^2\sigma_H^2)t^{2k-1}/\lambda.
    &\leq \frac{4(k+1)^2}{\lambda}(\sigma_G^2+D^2\sigma_H^2)t^{2k-1}.
\end{align*}
\end{proof}

\begin{restatable}{lemma}{LemmaPrivacyBound}
\label{lem:privacy-expectation}
Suppose $\E[R_t]=0$ and $\E[\|R_t\|_*^2]\leq \bar\sigma_t^2$, then
\begin{align*}
    \E[\|\gamma_t\|_*^2]
    &\leq 2(\max_{i\leq t}\bar\sigma_i^2)\log_2(2t)/\lambda.
\end{align*}
\end{restatable}

\begin{proof}
By construction, $\gamma_t=\sum_{i\in I_t}R_i$, and $R_i$'s are independent and mean-zero. Therefore, Lemma \ref{lem:martingale-bound-general-norm} can be applied, which yields
\begin{align*}
    \E[\|\gamma_t\|_*^2]
    &\leq \frac{2}{\lambda}\sum_{i\in I_t} \E[\|R_i\|_*^2] 
    \leq \frac{2}{\lambda} \sum_{i\in I_t} \bar\sigma_i^2
    \leq 2(\max_{i\leq t}\bar\sigma^2) \log_2(2t)/\lambda.
\end{align*}
The last inequality is from the fact that $|I_t|\leq \log_2(2t)$.
\end{proof}

\section{Proofs for RDP (Section \ref{sec:OTB})}
\label{app:tree-aggregation-RDP}

In this section, we prove the tree aggregation mechanism for RDP mechanisms implemented in Algorithm \ref{alg:DP-OTB} correctly composes individual RDP mechanisms. Before that, we will first prove a general composition theorem for RDP. 

\subsection{Advanced Composition for RDP}

Throughout this section, we use the subscript $1:T$ to denote a sequence of $T$ elements. 
We denote $\Z$ the data space and $Z=(z_1,\ldots,z_T), Z'=(z_1',\ldots,z_T')$ neighboring datasets in $\Z^T$ that differs only at the $q$-th element (i.e., $z_t\neq z_t'$ if and only if $t=q$).
% and we denote $W\subset\RR^d$ the range of 
We consider RDP mechanisms $F_1,\ldots,F_T$ such that $F_1:\Z^T\to W$ and $F_t:\Z^T\times W^{t-1}\to W$.  We assume that for each $t$, there exists some index set $S_t\subseteq[T]$ such that $F_t$ only depends on $S_t$. Formally, we assume:

\begin{assumption}
\label{as:data-dependence}
Let $Z,Z'\in\Z$ be two neighboring datasets which differs only at the $q$-th element. Each $F_t$ associates with $S_t\subseteq [T]$ such that if $q\not\in S_t$, then $F_t(Z,x_{1:t-1})=F_t(Z',x_{1:t-1})$ for all $x_{1:t-1}\in W^{t-1}$.
\end{assumption}

For a fixed norm $\|\cdot\|$, we assume $\D$ is a $(V,\alpha)$-RDP distribution on norm $\|\cdot\|$ (see Definition \ref{def:RDP-dist}), and we define the sensitivity of $F_t$ w.r.t. $\|\cdot\|$ as $\Delta_t(x_{1:t-1})$, a function of inputs $x_{1:t-1}$:
\begin{align*}
    % \Delta_t = \sup_{\substack{|Z-Z'|=1 \\ x_{1:t-1}\in W^{t-1}}} \|F_t(Z, x_{1:t-1}) - F_t(Z', x_{1:t-1})\|_*. \\
    \Delta_t(x_{1:t-1}) = \sup_{|Z-Z'|=1} \|F_t(Z,x_{1:t-1}) - F_t(Z',x_{1:t-1})\|_*.
\end{align*}
We also define the output as $\hat{f}_t = F_t(Z,\hat{f}_{1:t-1})+\sigma_t\zeta_t$, where $\zeta_t\sim \D$ and $\sigma_t^2\geq \Delta_t(\hat{f}_{1:t-1})^2/\rho^2$. In particular, $\sigma_t$ only depends on $\hat f_{1:t-1}$ and does not depend on $\hat f_{t:T}$, i.e., the future.
The pseudo-code of this composition is in Algorithm \ref{alg:gaussianaggregation}. For simplicity, we assume $F$'s are deterministic mechanisms, but we can extend to random mechanisms by treating the random generator as part of the input.

\begin{algorithm}
   \caption{Advanced Composition for RDP Mechanisms}
   \label{alg:gaussianaggregation}
   \begin{algorithmic}[1]
      \State{\bfseries Input: } Dataset $Z$; functions $F_1,\dots,F_T$ with sensitivity $\Delta_1,\ldots,\Delta_T$; $(V,\alpha)$-RDP distribution $\D$; privacy constants $\rho_1,\ldots,\rho_T$.
      \State Sample random $\zeta_1\sim \D$ and compute $\sigma_1^2 \geq \Delta_1^2/\rho_1^2$.
      \State Set $f_1 = F_1(Z)$ and $\hat f_ 1 = f_1 + \sigma_1\zeta_1$
      \For{$t=2,\dots,T$}
      \State Sample random $\zeta_t\sim \D$ and compute $\sigma_t^2\geq \Delta_t(\hat{f}_{1:t-1})^2/\rho_t^2$.
      \State Set $f_t =F_t(Z, \hat f_{1:t-1})$ and $\hat f_t = f_t + \sigma_t\zeta_t$.
      \EndFor
      \State \textbf{Return} $\hat f_1,\dots,\hat f_T$.
   \end{algorithmic}
\end{algorithm}

By definition of RDP distribution, each $\hat{f}_t$ is $(\alpha,\alpha\rho^2/2)$-RDP. We will also show that the composition $(\hat{f}_1,\ldots,\hat{f}_T)$ is also RDP.

\begin{theorem}\label{thm:renyiaggregation}
% Let $V$ to be the maximum over all $q\in\{1,\dots, N\}$ of the total number of sets $S_i$ such that $q\in S_i$ (i.e. $V=\sup_q |\IN(q)|$). Then Algorithm \ref{alg:gaussianaggregation} is $(\alpha, V \alpha\rho^2/2)$ Renyi differentially private for all $\alpha$.
We define $\IN(q)=\{t:q\in S_t\}$ and $\OUT(q)=\{t:q\not\in S_t\}$.
% , and define $V=\max_q |\IN(q)|$. 
If $F_1,\ldots,F_T$ satisfy Assumption \ref{as:data-dependence}, then
Algorithm \ref{alg:gaussianaggregation} is 
$(\alpha, S)$-RDP, where
\begin{align*}
    S = \max_{q\in[T]} \sum_{t\in \IN(q)} \alpha\rho_t^2/2.
\end{align*}
% $(\alpha, V\alpha\rho^2/2)$-RDP.
\end{theorem}

As an immediate corollary, if we set $\rho_t=\rho$ for all $t$ and define $U=\max_q |\IN(q)|$, then Algorithm \ref{alg:gaussianaggregation} is $(\alpha, U\alpha\rho^2/2)$-RDP.

\begin{proof}
Let $Z,Z'$ be any neighboring datasets and assume they differ at $q$, and we denote $\hat f_t = F_t(Z,\hat f_{1:t-1})+\sigma_t(\hat f_{1:t-1})\zeta_t$ and $\hat f'_t = F_t(Z',\hat f'_{1:t-1})+\sigma_t(\hat f'_{1:t-1})\zeta_t$. In this proof, we use the notation $\sigma_t(\hat f_{1:t-1})$ to emphasize that $\sigma_t$ satisfying $\sigma_t^2\geq \Delta_t(\hat f_{1:t-1})^2/\rho_t^2$ is a function of $\hat f_{1:t-1}$.
% We abbreviate sequence of random variables $V_1,\ldots,V_n$ as $V_{1:n}$ and sequence of constants $x_1,\ldots,x_n$ as $x_{1:n}$. 
% Recall that $\hat{f}_t\sim \D(f_t,\sigma_t^2)$. 
% Consider a fixed sequence of outputs $x_{1:T}$.

The probability density of the joint distribution of $\hat{f}_{1:T}$, say $P$, and the density of $\hat f'_{1:T}$, say $Q$, are:
\begin{align*}
    P(x_{1:T}) = \prod_{t=1}^T P_t(x_t|x_{1:t-1}), \quad Q(x_{1:T}) = \prod_{t=1}^T Q_t(x_t|x_{1:t-1}),
\end{align*}
% \begin{align*}
    % P(x_{1:T})
    % &= \P\{\hat f_1=r_1\}\P\{\hat f_2=r_2|\hat f_1=r_1\} \dotsm \P\{\hat f_T=r_T|\hat f_{1:T-1} = r_{1:T-1}\} \\
    % &= \prod_{t=1}^T \P\{ \hat f_t=x_t | \hat f_{1:t-1} = x_{1:t-1} \}.
    % \intertext{Recall that $\hat f_t = F_t(Z,\hat f_{1:t-1})+\sigma_t(\hat f_{1:t-1})\zeta_t$, and $\zeta_t\sim \D$. Therefore,}
    % &= \prod_{t=1}^T \D(f_t,\sigma_t^2|\hat f_{1:t-1}=x_{1:t-1})(x_t) = \prod_{t=1}^T \D(F_t(Z,x_{1:t-1}), \sigma_t^2)(x_t).
    % &= \prod_{t=1}^T \P\{ F_t(Z,x_{1:t-1})+\sigma_t(\hat f_{1:t-1})\zeta_t = x_t \}.
    % =\prod_{t=1}^T \P\left\{ \zeta_t = \frac{x_t-F_t(Z,x_{1:t-1})}{\sigma_t(x_{1:t-1})} \right\}.
% \end{align*}
where $P_t(\cdot|x_{1:t-1})$ is the density of $(\hat f_t | \hat f_{1:t-1}=x_{1:t-1}) = F_t(Z,x_{1:t-1})+\sigma_t(x_{1:t-1})\zeta_t$ and $Q_t(\cdot|x_{1:t-1})$ is the density of $F_t(Z',x_{1:t-1})+\sigma_t(x_{1:t-1})\zeta_t$. Note that $\sigma_t$ is the same for $P_t$ and $Q_t$.
% \begin{align*}
%     P(x_{1:T}) = \prod_{t=1}^T P_t(x_t|x_{1:t-1}), \quad Q(x_{1:T}) = \prod_{t=1}^T Q_t(x_t|x_{1:t-1}).
% \end{align*}

By definition of R\'enyi divergence,
\begin{align}
    D_\alpha(P\|Q) = \frac{1}{\alpha-1}\log \int P(x_{1:T})^\alpha Q(x_{1:T})^{1-\alpha}\, dx_{1:T}. \label{eq:RDP-renyi-divergence-defn}
\end{align}
Previous multiplication rule implies that:
\begin{align*}
    P(x_{1:T})^\alpha Q(x_{1:T})^{1-\alpha}
    = \prod_{t\in \IN(q)\sqcup \OUT(q)} \left( P_t(x_t|x_{1:t-1})^\alpha Q_t(x_t|x_{1:t-1})^{1-\alpha}\right) 
\end{align*}
For all $t\in\OUT(q)$ (i.e., $q\not\in S_t$), Assumption \ref{as:data-dependence} implies that $F_t(Z,x_{1:t-1})=F_t(Z',x_{1:t-1})$, so $P_t(\cdot|x_{1:t-1}) = Q_t(\cdot|x_{1:t-1})$ and
\begin{align}
    \int P_t(x_t|x_{1:t-1})^\alpha Q_t(x_t|x_{1:t-1})^{1-\alpha}\, dx_t
    = \int P_t(x_t|x_{1:t-1}) \, dx_t = 1.
    \label{eq:RDP-composition-1}
\end{align}
The second inequality holds because $P_t(\cdot|x_{1:t-1})$ is a probability density. 

On the other hand, for all $t\in \IN(q)$,
\begin{align*}
    \sigma_t(x_{1:t-1})^2
    &\geq \Delta_t(x_{1:t-1})^2/\rho_t^2 \\
    &\geq \|F_t(Z,x_{1:t-1})-F_t(Z',x_{1:t-1})\|_*^2/\rho_t^2,
\end{align*}
so by Definition \ref{def:RDP-dist},
\begin{align*}
    D_\alpha(P_t\|Q_t)
    % &= D_\alpha(\D(F_t(Z,x_{1:t-1}),\sigma_t^2)\|\D(F_t(Z',x_{1:t-1}),\sigma_t^2)) \\
    &= \frac{1}{\alpha-1}\log\int P_t(x|x_{1:t-1})^\alpha Q_t(x|x_{1:t-1})^{1-\alpha} \, dx
    \leq \alpha\rho_t^2/2.
\end{align*}
Equivalently,
\begin{align}
    \int P_t(x|x_{1:t-1})^\alpha Q_t(x|x_{1:t-1})^{1-\alpha} \, dx \leq \exp((\alpha-1)\alpha\rho_t^2/2).
    \label{eq:RDP-composition-2}
\end{align}

Note that $P_t,Q_t$ only depend on $x_{1:t-1}$ and \textit{not} on $x_{t:T}$, so we can rearrange the integral in \eqref{eq:RDP-renyi-divergence-defn} as:
\begin{align*}
    &\int P(x_{1:T})^\alpha Q(x_{1:T})^{1-\alpha} \, dx_{1:T} \\
    = &\int P_T(x_T|x_{1:T-1})^\alpha Q_T(x_T|x_{1:T-1})^{1-\alpha} \dotsm \left(\int P_1(x_1)^\alpha Q_1(x_1)^{1-\alpha} \, dx_1\right) \dotsm dx_T \\
    \intertext{Evaluating the composite integral from inside to outside with \eqref{eq:RDP-composition-1} and \eqref{eq:RDP-composition-2} gives:}
    \leq & \exp\left(\sum_{t\in\IN(q)} (\alpha-1)\alpha\rho_t^2/2\right).
\end{align*}
In conclusion, for all $|Z-Z'|=1$,
\begin{align*}
    D_\alpha(P\|Q) = \frac{1}{\alpha-1} \log \int P(x_{1:T})^\alpha Q(x_{1:T})^{1-\alpha}\, dx_{1:T} \leq \max_q \sum_{t\in \IN(q)}\alpha\rho_t^2/2 = S.
\end{align*}
\end{proof}

\subsection{Algorithm \ref{alg:DP-OTB} is RDP}

Now we are ready to prove the tree aggregation in Algorithm \ref{alg:DP-OTB}. 

\TreeRDP*

\begin{proof}
Recall the definition of $I_t$ in Algorithm \ref{alg:DP-OTB}: we define $s_0=0$, and $s_i=\max_k\{s_{i-1}+2^k:2^k|t-s_{i-1}\}$ until $s_n=t$ for some $n$, and we define $I_t=\{s_1,\ldots,s_n\}$. For example, $I_4=\{4\}$ and $I_7=\{4,6,7\}$. We then define $S_t=\{s_{n-1}+1,s_{n-1}+2,\ldots, t\}$ (e.g., $S_4=\{1,2,3,4\}$ and $S_7=\{7\}$). Observe that $\{S_i:i\in I_t\}$ is a partition of $[t]$.

Let $Z,Z'$ be neighboring datasets that differ at the $q$-th element. 
Define $F_t:\Z^T\times W^{t-1}\to W$ as:
\begin{align*}
    F_t(Z,\hat f_{1:t-1}) 
    &= \sum_{i\in S_t} \delta_i(Z, \hat f_{1:t-1}) \\
    &= \sum_{i\in S_t} \weight_i\nabla\ell\left(x_i,z_i\right)-\weight_{i-1}\nabla\ell\left(x_{i-1},z_i\right),
\end{align*}
where $x_i$'s are the parameters as defined in Algorithm \ref{alg:DP-OTB}, and $z_i$ is the $i$-th data in $Z$. We then define $\hat f_t = F_t(Z,\hat f_{1:t-1}) + \sigma_t(\hat f_{1:t-1})\tilde R_t$ and $\hat f_t' = F_t(Z',\hat f'_{1:t-1}) + \sigma_t(\hat f'_{1:t-1})\tilde R_t$, where $\tilde R_t\sim\D$. 

For simplicity, we denote $\delta_t=\delta_t(Z,\hat f_{1:t-1})$ and $\sigma_t=\sigma_t(\hat f_{1:t-1})$ (and $\delta_t',\sigma_t'$ respectively). We also denote $x_i,w_i$ and $x_i',w_i'$ as parameters w.r.t. $Z,Z'$ respectively.
% We need to verify that $F_t$ is well-defined, i.e., we prove that $x_{1:t}$ (defined as in Algorithm \ref{alg:DP-OTB}) are deterministic functions of $\hat f_{1:t-1}$ by induction. 
% We can prove by induction that function $x_i(Z,\hat f_{1:t-1})$ is well-defined for all $i\in[t]$. For the base case, $x_1=w_1$, which is a fixed initial point. Now assume $x_{1:t}$ are deterministic functions of $\hat{f}_{1:t-1}$ and thus $\hat{f}_{1:t}$ are well-defined. Then 
Since $\{S_i:i\in I_t\}$ partitions $[t]$,
\begin{align*}
    g_t+\gamma_t = \sum_{i=1}^t \delta_i + \sum_{i\in I_t} R_i
    = \sum_{i\in I_t} \left(\sum_{j\in S_i} \delta_j + \sigma_i \tilde R_i\right)
    = \sum_{i\in I_t} \hat f_i.
\end{align*}
By construction of Algorithm \ref{alg:DP-OTB}, $x_1,\ldots,x_t$ are determined by $\{g_i+\gamma_i\}_{i=1}^{t-1}$ and equivalently by $\hat f_{1:t-1}$. In particular, they do not depend on $f_{t:T}$. Consequently, this implies that (i) if $\hat f_{1:t-1}=\hat f'_{1:t-1}$ then $x_{1:t}=x'_{1:t}$ and $w_{1:t}=w'_{1:t}$ and (ii) if in addition $z_i=z_i'$ then $\delta_i=\delta_i'$.

This implies that $F_t$'s satisfy Assumption \ref{as:data-dependence}: if $q\not\in S_t$ (i.e., $z_i=z_i'$ for all $i\in S_t$), then for any fixed $f_{1:t-1}\in W^{t-1}$, $F_t(Z, f_{1:t-1}) = F_t(Z', f_{1:t-1})$ because $\delta_i=\delta_i'$ for all $i\in S_t$. Consequently, Theorem \ref{thm:renyiaggregation} can be applied, which states that if $\sigma_t(\hat f_{1:t-1})^2\geq \Delta_t(\hat f_{1:t-1})^2/\rho^2$, then $(\hat f_1,\ldots, \hat f_T)$ is $(\alpha, U\alpha\rho^2/2)$-RDP where $U=\max_q|\IN(q)|$ and $\IN(q)=\{t:q\in S_t\}$. Note that $U\leq \log_2(2T)$.

The sensitivity of $F_t$ at fixed $f_{1:t-1}\in W^{t-1}$ is bounded by:
\begin{align*}
    \Delta_t(f_{1:t-1})
    &= \sup_{\substack{|Z-Z'|=1}} \|F_t(Z, f_{1:t-1}) - F_t(Z', f_{1:t-1})\|_* \\
    &= \sup_{q\in S_t}\|\delta_q - \delta_q'\|_*
    \leq \sup_{q\in S_t} \|\delta_q\|_* + \sup_{q\in S_t}\|\delta_q'\|_*
    \intertext{We proved that $\|\delta_i\|_*\leq (k+1)i^{k-1}(G+H\|w_i-x_{i-1}\|)$ and $\|\delta_i'\|_*\leq (k+1)i^{k-1}(G+H\|w_i'-x'_{i-1}\|)$ (Lemma \ref{lem:variance-expectation-bound}). Note that $w_i=w_i'$ and $x_i=x_i'$ for all $i\leq t$. Also note that $i\leq t$ for all $i\in S_t$, so:}
    &\leq 2(k+1)t^{k-1}(G+H\max_{i\in[t]}\|w_i-x_{i-1}\|).
\end{align*}
Since $U\leq \log_2(2T)$ and $\sigma_t^2$ as defined in \eqref{eq:RDP-noise} satisfies the condition $\sigma_t^2 \geq \Delta_t(\hat f_{1:t-1})\log_2(2T)/\rho^2$, Theorem \ref{thm:renyiaggregation} and post-processing imply that $\{g_t+\gamma_t\}_{t=1}^T$ is $(\alpha,\alpha\rho^2/2)$-RDP (so is Algorithm \ref{alg:DP-OTB}).
\end{proof}

\section{Further Discussions about Differential Privacy}
\label{app:further-DP}

\subsection{Example of RDP Distribution}

In this subsection, we prove that the multivariate Gaussian distribution $\N(0,I)$ is a $(d,\alpha)$-RDP distribution w.r.t. 2-norm on $\RR^d$ for all $\alpha>0$ (Definition \ref{def:RDP-dist}). Namely, $\N(0,I)$ satisfies the following three properties: let $R\sim \N(0,I)$, then (i) $\E[R]=0$, (ii) $\E[\|R\|_2^2] \leq d$, and (iii) for all $\rho>0$ and $\mu,\mu'\in\RR^d$, if $\sigma^2 \geq \|\mu-\mu'\|_2^2/\rho^2$ then $D_\alpha(P\|Q) \leq \alpha\rho^2/2$, where $P,Q$ denote the distribution of $\sigma R+\mu$ and $\sigma R+\mu'$ respectively.

The first property follows immediately from the definition of $\N(0,I)$. For the second property, $R=(r_1,\ldots,r_d)$ where $r_i\sim N(0,1)$ iid., so $\E[\|R\|_2^2]=\sum_{i=1}^d \E[r_i^2]=d$. To check the third property, we need the following lemma:

\begin{lemma}\label{lemma:renyigaussian}
$D_\alpha(\N(\mu,\sigma^2I)\| \N(\mu',\sigma^2I))=\alpha \|\mu-\mu'\|^2/2\sigma^2$.
\end{lemma}

Consequently, for all $\sigma^2 \geq \|\mu-\mu'\|_2^2/\rho^2$, $D_\alpha(\N(0,\sigma^2I)\| \N(\mu, \sigma^2 I)) \leq \alpha\rho^2/2$. This proves that $\N(0,I)$ is indeed a $(d,\alpha)$-RDP distribution.

\begin{proof}[Proof of Lemma \ref{lemma:renyigaussian}]
The density of $\N(\mu,\sigma^2 I)$ is $(2\pi\sigma^2)^{-d/2}\exp(-\|x-\mu\|_2^2/2\sigma^2)$. For short we denote $A=(2\pi\sigma^2)^{-d/2}$ and $B=1/2\sigma^2$. Then
\begin{align*}
    &D_\alpha(\N(\mu,\sigma^2 I)\|\N(\mu',\sigma^2 I))\\
    &=\frac{1}{\alpha-1}\log\left(\int_{\R^d} A^\alpha\exp(-B\alpha\|x-\mu\|_2^2) A^{1-\alpha}\exp(-B(1-\alpha)\|x-\mu'\|_2^2)\ dx\right) \\
    &= \frac{1}{\alpha-1}\log \left( \int_{\RR^d} A\exp \left(-B(\alpha\|x-\mu\|_2^2+(1-\alpha)\|x-\mu'\|_2^2) \right) dx \right).
\end{align*}
Next, observe that
\begin{align*}
    & x-\mu= (x - \alpha\mu - (1-\alpha)\mu') -(1-\alpha)(\mu-\mu'), \\
    & x-\mu' = (x-\alpha\mu-(1-\alpha)\mu')+\alpha(\mu-\mu').
\end{align*}
Consequently, upon expanding out $\|x-\mu\|_2^2, \|x-\mu'\|_2^2$, we get:
\begin{align*}
    \alpha\|x-\mu\|_2^2 + (1-\alpha)\|x-\mu'\|_2^2
    = \|x-\alpha\mu-(1-\alpha)\mu'\|_2^2 + \alpha(1-\alpha)\|\mu-\mu'\|_2^2.
\end{align*}
Note that $A\exp(-B\|x-\alpha\mu-(1-\alpha)\mu'\|_2^2)$ is the density of $\N(\alpha\mu+(1-\alpha)\mu', \sigma^2I)$, so it integrates to $1$. Therefore,
\begin{align*}
    &D_\alpha(\N(\mu,\sigma^2 I)\|\N(\mu',\sigma^2 I)) \\
    &= \frac{1}{\alpha-1}\log\left( \int_{\RR^d} A\exp(-B\|x-\alpha\mu-(1-\alpha)\mu'\|_2^2) \exp(-B\alpha(1-\alpha)\|\mu-\mu'\|_2^2) \ dx \right) \\
    &= \frac{1}{\alpha-1}\log\left( \exp(-B\alpha(1-\alpha)\|\mu-\mu'\|_2^2) \right)
    = B\alpha\|\mu-\mu'\|_2^2.
\end{align*}
Recall that $B=1/2\sigma^2$, and this completes the proof.
\end{proof}

\subsection{Extension to Pure-DP Mechanisms}

In the main text, we focus on Renyi differential privacy, and we defined RDP-distribution (Definition \ref{def:RDP-dist}) accordingly. We can always extend our result in a pure differential privacy setting.

\newcommand{\DPnorm}{\text{$(K,\D)$-DP-able}}
\begin{definition}[$V$-DP distribution]
\label{def:DP-dist}
A distribution $\D$ on $\RR^d$ is said to be a \textit{DP distribution on norm $\|\cdot\|$ with variance constant $V$} (or simply $\D$ is $V$-DP on $\|\cdot\|$) if it satisfies that for $R\sim\D$ (i) $\E[R]=0$, (ii) $\E[\|R\|_*^2]\leq V$, and (iii) for all $\epsilon>0$ and $\mu,\mu'\in\RR^d$, if $\sigma^2\geq \|\mu-\mu'\|_*^2/\epsilon^2$, then $p((x-\mu)/\sigma)/p((x-\mu')/\sigma) \leq \exp(\epsilon)$ for all $x\in\RR^d$, where $p(x)$ is the density of $\D$.
\end{definition}

The tree aggregation described in Appendix \ref{app:tree-aggregation-RDP} also works for pure DP mechanisms as well. Therefore, if we assume $\D$ in Algorithm \ref{alg:DP-OTB} with a $V$-DP distribution and change the definition of $\sigma_t^2$ in \eqref{eq:RDP-noise} correspondingly, Algorithm \ref{alg:DP-OTB} can be modified to an purely $\epsilon$-DP mechanism.

Next, we can show that exponential mechanism in general norm satisfies this definition:

\begin{restatable}{theorem}{TheoremExponentialMechanism}
\label{thm:exponential-mechanism-general-norm}
Consider a probability density $p(x) = A\exp(-\|x\|_*)$ on $(\RR^d,\|\cdot\|)$, where $A$ is some normalization constant. Also define $V=\int_{\RR^d} \|x\|_*^2 A\exp(-\|x\|_*)\, dx$. Then distribution $\D$ with density $p$ is a $\K$-DP distribution.
\end{restatable}

\begin{proof}
Let $R\sim \D, \mu,\mu'\in \RR^d$, and $\sigma^2\geq 0$. Since the density $p$ is symmetric, $\E[R]=0$; and by definition, $\E[\|R\|_*^2]=V$. For the third property,
\begin{align*}
    \frac{p((x-\mu)/\sigma)}{p((x-\mu')/\sigma)} 
    = \frac{A\exp(-\|x-\mu\|_*/\sigma)}{A\exp(-\|x-\mu'\|_*/\sigma)}
    &= \exp\left( \frac{-\|x-\mu\|_*+\|x-\mu'\|_*}{\sigma} \right) \\
    \intertext{By triangular inequality, $-\|x-\mu\|_*+\|x-\mu'\|_*\leq \|\mu-\mu'\|_*$, so:}
    &\leq \exp\left(\frac{\|\mu-\mu'\|_*}{\sigma}\right).
\end{align*}
Hence, for all $\sigma \geq \|\mu-\mu'\|_*/\epsilon$, this is further bounded by $\exp(\epsilon)$.
% For the second part, we quote \cite{folland1999real} Corollary 2.51: if $f$ is measurable on $\RR^d$ and integrable such that $f(x)=g(\|x\|_2)$ for some function $g$ on $(0,\infty)$, then 
% \[
% \int_{\RR^d} f(x) \, dx = A(S^{d-1}) \int_0^\infty g(r)r^{d-1} \, dr.
% \]
% Therefore, 
% \begin{align*}
%     1 &= \int_{\RR^d} C\exp\left(-\frac{\|x-f(D)\|}{\Delta/\epsilon}\right) \, dx \\
%     &= A(S^{d-1})\int_0^\infty Cr^{d-1}\exp\left(-\frac{r}{\Delta/\epsilon}\right) \, dx \\
%     &= A(S^{d-1}) C (\Delta/\epsilon)^d\Gamma(d).
% \end{align*}
% Similarly,
% \begin{align*}
%     \E[\|\hat{f}(D)-f(D)\|_*^2]
%     &= \int_{\RR^d} \|x-f(D)\|_*^2 C\exp\left(-\frac{\|x-f(D)\|_*}{\Delta/\epsilon}\right) \, dx \\
%     &= A(S^{d-1}) \int_0^\infty C r^{d+1}\exp\left(-\frac{r}{\Delta/\epsilon}\right) \, dr \\
%     &= A(S^{d-1}) C(\Delta/\epsilon)^{d+2}\Gamma(d+2) \\
%     &= \frac{(\Delta/\epsilon)^2\Gamma(d+2)}{\Gamma(d)} 
%     = \frac{d(d+1)\Delta^2}{\epsilon^2}.
% \end{align*}
% \highlight{issue: change of variable w.r.t. $\|\cdot\|$.}
\end{proof}

\section{Proofs for the Optimistic Case (Section \ref{sec:optimistic})}
\label{app:optimistic}

\ThmOptimistic*

\begin{proof}
Recall that $g_t=\sum_{i=1}^t \delta_i$, then
\begin{align*}
    \|\Bar{g}_t - \hat{g}_t\|_*^2
    &= \|\delta_t + \gamma_t - \gamma_{t-1}\|_*^2
    \leq 3\|\delta_t\|_*^2 + 3\|\gamma_t\|_*^2 + 3\|\gamma_{t-1}\|_*^2.
\end{align*}
We showed (Lemma \ref{lem:variance-expectation-bound}) that
\begin{align*}
    \|\delta_t\|_* 
    % &\leq kt^{k-1}G+(k+1)t^{k-1}DH \\
    \leq (k+1)(G+H\|w_t-x_{t-1}\|)t^{k-1}
    \leq (k+1)(G+DH)t^{k-1}.
\end{align*}
Also recall the bound of $\E[\|\gamma_t\|_*^2]$ in \eqref{eq:gamma-bound}, so:
\begin{align*}
    \E[\|\bar{g}_t-\hat{g}_t\|_*^2]
    &\leq 3(k+1)^2(G+DH)^2t^{2k-2} + \frac{48(k+1)^2V(G+DH)^2}{\lambda\rho^2}\log_2^2(2T)t^{2k-2} \\
    &= 3(k+1)^2(G+DH)^2t^{2k-2}\left(1+\frac{16V\log_2^2(2T)}{\lambda\rho^2}\right).
\end{align*}
Recall that $\E[\reg_T(x^*)] \leq O(\E[D\sqrt{\sum_{t=1}^T\|\bar g_t-\hat g_t\|_*^2}])$. By Jensen's inequality,
\begin{align*}
    \E\left[D\sqrt{\sum_{t=1}^T\|\Bar{g}_{t}-\hat{g}_t\|_*^2}\right] 
    &\leq D\sqrt{\sum_{t=1}^T \E[\|\Bar{g}_t-\hat{g}_t\|_*^2]} \\
    &\leq \sqrt{3}(k+1)D(G+DH)\left(1+\frac{4\sqrt{V}\log_2(2T)}{\sqrt{\lambda}\rho}\right)T^{k-1/2}.
\end{align*}
% For the second part, observe that $\sigma_G\leq G, \sigma_H\leq H$, and thus $\frac{(k+1)\E[R_T(x^*)]}{T^{k+1}}$ is dominated by the other terms in Theorem \ref{thm:DP-OTB-general-norm}.
% \highlight{use big-O notation or put the whole thing (inequality is not true for small $T$}
Finally, dividing this bound by $\weight_{1:T}\geq T^{k+1}/(k+1)$ completes the proof.
\end{proof}

\section{Proofs for the Strongly Convex Case (Section \ref{sec:strongly-convex})}
\label{app:strongly-convex}

\begin{lemma}
\label{lem:basic-lemma-sc}
For any sequence $\weight_t>0, g_t\in\RR^d$, suppose an online learner predicts $w_t$ and receives $t$-th loss $\ell_t(w)=\langle g_t,w\rangle$, and define $x_t=\sum_{i=1}^t\frac{\weight_iw_i}{\weight_{1:t}}$. If $\L$ is $\mu$-strongly convex w.r.t. $\|\cdot\|$, then
\begin{align*}
    \weight_{1:T}(\L(x_T) - \L(x^*)) \leq \reg_T(x^*) + \sum_{t=1}^T \left( \langle \weight_t \nabla\L(x_t) - g_t, w_t-x^* \rangle - \frac{\weight_t\mu}{2}\|x_t-x^*\|^2 \right).
\end{align*}
\end{lemma}

\begin{proof}
We start with the strong convexity identity $\L(x^*)\geq\L(x_t)+\langle\nabla\L(x_t),x^*-x_t\rangle + \frac{\mu}{2}\|x_t-x^*\|^2$:
\begin{align}
    \sum_{t=1}^T\weight_t(\L(x_t) - \L(x^*))
    &\leq \sum_{t=1}^T \weight_t\langle \nabla\L(x_t),x_t-x^*\rangle - \frac{\weight_t\mu}{2}\|x_t-x^*\|^2.
    \label{eq:sc-1}
\end{align}
With the same argument in the proof of Lemma \ref{lem:basic-lemma}, we can show:
\begin{align*}
    \weight_t\langle \nabla\L(x_t),x_t-x^*\rangle
    &= \weight_t \langle \nabla\L(x_t), x_t-w_t \rangle + \weight_t \langle \nabla\L(x_t), w_t-x^*\rangle 
    \intertext{Recall the definition that $\weight_{1:t}x_t = \weight_{1:t-1}x_{t-1}+\weight_tw_t$ and thus $\weight_t(x_t-w_t)=\weight_{1:t-1}(x_{t-1}-x_t)$. Also, since $\L$ is convex, $\langle \nabla\L(x_t),x_{t-1}-x_t\rangle \leq \L(x_{t-1})-\L(x_t)$, so:}
    &\leq \weight_{1:t-1}\L(x_{t-1}) - \weight_{1:t-1}\L(x_t) + \langle \weight_t\nabla\L(x_t) - g_t + g_t,w_t-x^*\rangle.
\end{align*}
Consequently, moving $\sum_{t=1}^T \weight_t\L(x_t)$ to the right side and taking the telescopic sum in \eqref{eq:sc-1} gives:
\begin{align*}
    -\weight_{1:T}\L(x^*)
    &\leq \sum_{t=1}^T \weight_t\langle \nabla\L(x_t),x_t-x^*\rangle -\weight_t\L(x_t) - \frac{\weight_t\mu}{2}\|x_t-x^*\|^2 \\
    &\leq -\weight_{1:T}\L(x_T) + \reg_T(x^*) + \sum_{t=1}^T \langle\weight_t\nabla\L(x_t)-g_t,w_t-x^*\rangle - \frac{\weight_t\mu}{2}\|x_t-x^*\|^2.
\end{align*}
Moving $\weight_{1:T}\L(x_T)$ to the left completes the proof.
\end{proof}

This lemma immediately implies Lemma \ref{lem:strongly-convex-regularized-loss}.

\LemmaSCBasic*

\begin{proof}
By definition, $\bar{\ell}_t(w)=\ell_t(w) + \frac{\weight_t\mu}{4}\|w-x_t\|^2$, so
\begin{align*}
    \overline \reg_T(x^*) 
    &= \sum_{t=1}^T \left( \ell_t(w_t)+\frac{\weight_t\mu}{4}\|w_t-x_t\|^2\right) - \left( \ell_t(x^*) + \frac{\weight_t\mu}{4}\|x^*-x_t\|^2 \right) \\
    &= \reg_T(x^*) + \sum_{t=1}^T \frac{\weight_t\mu}{4}(\|w_t-x_t\|^2-\|x_t-x^*\|^2).
\end{align*}
Upon substituting this equation into Lemma \ref{lem:basic-lemma-sc}, we get:
\begin{align*}
    &\quad \ \weight_{1:T}(\L(x_T) - \L(x^*)) \\
    &\leq \overline\reg_T(x^*) -\sum_{t=1}^T \frac{\weight_t\mu}{4}(\|w_t-x_t\|^2-\|x_t-x^*\|^2) \\
    &+ \sum_{t=1}^T \langle \weight_t \nabla\L(x_t)-g_t-\gamma_t, w_t-x^*\rangle - \frac{\weight_t\mu}{2} \|x_t-x^*\|^2 \\
    &\leq \overline\reg_T(x^*) +\sum_{t=1}^T \langle \weight_t\nabla\L(x_t)-g_t-\gamma_t, w_t-x_t\rangle - \frac{\weight_t\mu}{4}(\|w_t-x_t\|^2+\|x_t-x^*\|^2) \\
    &\leq \overline\reg_T(x^*) +\sum_{t=1}^T \langle\weight_t\nabla\L(x_t)-g_t-\gamma_t, w_t-x^*\rangle - \frac{\weight_t\mu}{8}\|w_t-x^*\|^2.
\end{align*}
The last inequality follows from the identity $\|w_t-x^*\|^2 \leq 2\|w_t-x_t\|^2 + 2\|x_t-x^*\|^2$.

For the second inequality in the lemma, by Fenchel-Young's inequality,
\begin{align*}
    &\langle\weight_t\nabla\L(x_t)-g_t-\gamma_t, w_t-x^*\rangle - \frac{\weight_t\mu}{8}\|w_t-x^*\|^2 \\
    \leq &\|\weight_t\nabla\L(x_t)-g_t-\gamma_t\|_*\|w_t-x^*\|-\frac{\weight_t\mu}{8}\|w_t-x^*\|^2 \\
    % \intertext{Denote $B_t=\|\weight_t\nabla\L(x_t)-g_t-\gamma_t\|_*$ for short, and note that this is a quadratic equation:}
    \intertext{For any quadratic of form $ax-bx^2$ and $a,b>0$, note that $\sup_x ax-bx^2\leq a^2/4b$, so:}
    \leq & \frac{2\|\weight_t\nabla\L(x_t)-g_t-\gamma_t\|_*^2}{\weight_t\mu}.
\end{align*}
\end{proof}

\begin{proposition}
\label{prop:subdiff-chain-rule}
Suppose $W$ is a convex bounded domain with diameter $D$, and let $u\in W$ and $f(w)=\|w-u\|^2$. Then for all $w\in W$ and $v\in \partial f(w)$, $\|v\|_*\leq 2D$.
\end{proposition}

\begin{proof}
Let $\phi(r)=r^2$ and $g(w)=\|w-u\|$, then $f(w) = \phi\circ g(w)$. By chain rule of sub-differentials (Corollary 16.72 \cite{bauschke2011convex}), 
\begin{align*}
    \partial f(w)
    &= \{\alpha v' : \alpha\in \partial \phi(g(w)), v'\in\partial g(w) \} \\
    &= \{2\|w-u\| v' : v'\in\partial \|w-u\|\}.
\end{align*}
By assumption, $\|w-u\|\leq D$. Moreover, $\|\cdot\|$ is $1$-Lipschitz (because $\|x\|-\|y\|\leq \|x-y\|$), so $\|v'\|_*\leq 1$ for all $v'\in\partial \|w-u\|$. As a result, for all $v\in \partial f(w)$, $\|v\|_* = 2\|w-u\|\|v'\|_* \leq 2D$.
\end{proof}
\section{Proofs for the Parameter-free Case (Section \ref{sec:parameter-free})}
\label{app:parameter-free}

% Previously, we always bound $\|w_t-x^*\|\leq D$, but the following lemma gives a finer bound on the inner product, which preserves the quantity $\|w_t\|$ and $\|x^*\|$. This allows us to come up with an adaptive bound by regularizing $\ell_t(w)$ by $\|w\|^2$.
% In this subsection, we only consider the specific case in 2-norm, and we make the following definition:
\begin{definition}
A random vector $X\in\RR^d$ is said to be $\sigma$-norm-sub-Gaussian, denoted by $\nSG(\sigma)$ if
\begin{align*}
    \P\{\|X-\E[X]\|_2 \geq \epsilon\} \leq 2\exp\left(-\frac{\epsilon^2}{2\sigma^2}\right), \ \forall \epsilon.
\end{align*}
\end{definition}

We will rely on the following concentration bound on norm-sub-Gaussian random vectors.

\begin{lemma}[Lemma 1, \cite{jin2019short}]
\label{lem:normSG}
There exists a universal $C$ such that (i) if $\|X\|\leq \sigma$, then $X$ is $\nSG(C\sigma)$ nad (ii) if $X$ is $\sigma$-sub-Gaussian, then $X$ is $\nSG(C\sqrt{d}\sigma)$. 
\end{lemma}

\begin{lemma}[Corollary 8, \cite{jin2019short}]
\label{lem:nSG-concentration}
There exists a universal constant $C$ such that if $X_i|X_{1:i-1}$ is mean-zero $\nSG(\sigma_i)$ for all $X_1,\ldots,X_t$, then for any fixed $\delta>0$ and $B>b>0$ such that $b<\sum_{i=1}^t \sigma_i^2 \leq B$ almost surely, with probability at least $1-\delta$,
\begin{align*}
    \left\|\sum_{i=1}^t X_i\right\|_2 \leq C\sqrt{\sum_{i=1}^t \sigma_i^2\left(\log\frac{2d}{\delta}+\log\log\frac{B}{b} \right)}.
\end{align*}
\end{lemma}

Recall that $\weight_t\nabla\L(x_t)-g_t=\sum_{i=1}^t X_i$ (Lemma \ref{lem:variance-expectation-bound}), where
\[
X_i = [\weight_i\nabla\L(x_i)-\weight_{i-1}\nabla\L(x_{i-1})]-[\weight_i\nabla\ell(x_i,z_i)-\weight_{i-1}\nabla\ell(x_{i-1},z_i)];
\]
and $\gamma_t = \sum_{i\in I_t}R_i$, where $R_i=\sigma_i\tilde R_i$ and $R_i\sim \D$ i.i.d. We have the following lemma:

\begin{lemma}
\label{lem:X-R-nSG}
Suppose Assumption \ref{as:bounded-domain} - \ref{as:smooth} hold w.r.t. the 2-norm, and suppose $\D$ is a $(V,\alpha)$-RDP distribution and is $\sigma_\D$-sub-Gaussian, i.e.,
\begin{align*}
    \P\{\sup_{\|a\|=1} \langle X, a\rangle \geq \epsilon \} \leq \exp\left(- \frac{\epsilon^2}{2\sigma_\D^2}\right).
\end{align*}
Also set $\weight_t=t^k$. Then there exists a universal constant $C$ such that $X_i|X_{1:i-1}$ are mean-zero $\nSG(\sigma_{X_i})$ and $R_i|R_{1:i-1}$ are mean-zero $\nSG(\sigma_{R_i})$ for all $i$, where
\begin{align*}
    % & \sigma_{X_i} = 2C|\weight_i-\weight_{i-1}|G+2C(\weight_i^2/\weight_{1:i})H\|w_i-x_{i-1}\|_2, \\
    & \sigma_{X_i} = 2C(k+1)(G+H\|w_i-x_{i-1}\|_2)i^{k-1}, \\
    & \sigma_{R_i} = C\sqrt{d}\sigma_\D \sigma_i.
\end{align*}
\end{lemma}

\begin{proof}
Since we assume $\E[\nabla\ell(x,z)]=\nabla\L(x)$ for all $x$, $\E[X_i|X_{1:i-1}]=0$. Also, since $\D$ is a $(V,\alpha)$-RDP distribution (Definition \ref{def:RDP-dist}) and $R_i=\sigma_i\tilde R_i$'s are independent, $\E[R_i|R_{1:i-1}]=\E[\tilde R_i]=0$. 

For the second part, in Lemma \ref{lem:variance-expectation-bound} we proved that
\begin{align*}
    \|\delta_i\|_2 = \|\weight_i\nabla\ell(x_i,z_i)-\weight_{i-1}\nabla\ell(x_{i-1},z_i)\|_2 \leq (k+1)(G+H\|w_i-x_{i-1}\|_2)i^{k-1}.
\end{align*}
The same bound holds for $\weight_i\nabla\L(x_i)-\weight_{i-1}\nabla\L(x_{i-1})$ following the same argument. Therefore,
\[
% \|X_i\|_2 \leq 2|\weight_i-\weight_{i-1}|G+2(\weight_i^2/\weight_{1:i})H\|w_i-x_{i-1}\|_2.
\|X_i\|_2 \leq 2(k+1)(G+H\|w_i-x_{i-1}\|_2)i^{k-1}.
\]
Moreover, since we assume $\tilde R_i\sim \D$ is $\sigma_\D$-sub-Gaussian, $R_i=\sigma_i\tilde R_i$ is $\sigma_i\sigma_\D$-sub-Gaussian.
Hence, by Lemma \ref{lem:normSG}, $X_i|X_{1:i-1}$ and $R_i|R_{1:i-1}$ are norm-sub-Gaussian.
\end{proof}

% \highlight{To be more careful in the proof, we may need to say $\F_t=\{(w_i,x_i)\}_{i=1}^{t+1}\cup\{(g_i,\gamma_i)\}_1^t$ is a filtration of $X_t$, then $X_i|\F_{t-1}$ is conditionally $\nSG(...\|w_i-x_{i-1}\|)$.}

\ThmParameterFreeA*

\begin{proof}
We start with Eq. \eqref{eq:basic-decomposition}:
\begin{align}
    \weight_{1:T}(\L(x_T)-\L(x^*))
    &\leq R_T(x^*) + \sum_{t=1}^T\langle \weight_t\nabla\L(x_t)-g_t-\gamma_t,w_t-x^*\rangle \notag\\
    % \intertext{By Fenchel-Young's inequality and triangle inequality,}
    &\leq R_T(x^*) + \sum_{t=1}^T \left(\|\weight_t\nabla\L(x_t)-g_t\|_2+\|\gamma_t\|_2\right) (\|w_t-x^*\|_2). \label{eq:parameter-free-base}
\end{align}

\textbf{Step 1.}
By Lemma \ref{lem:nSG-concentration} and \ref{lem:X-R-nSG}, for each $t$, with probability $1-\delta/2T$,
\begin{align*}
    \|\weight_t\nabla\L(x_t)-g_t\|_2 
    \leq C\sqrt{\sum_{i=1}^t \sigma_{X_i}^2 \left( \log\frac{4dT}{\delta} + \log\log\frac{B}{b}\right)}.
\end{align*}
Since we choose $\weight_t=t^3$ (i.e., $k=3$),
\begin{align*}
    \sigma_{X_i} = 8Ci^2(G+H\|w_i-x_{i-1}\|_2).
\end{align*}
% Note that for all $t$, $\sum_{i=1}^t \sigma_{X_i}^2 \leq B:=64C^2(G+DH)^2T^5$, and $\sum_{i=t}^t\sigma_{X_i}^2\geq b:=\sigma_{X_1}^2=4C^2(G+H\|w_1\|_2)^2$. 
Next, we can bound $\sum_{i=1}^t\sigma_{X_i}^2$ as follows: for all $t$,
\begin{align*}
    \textstyle \sum_{i=1}^t \sigma_{X_i}^2 &\leq B := 64C^2(G+DH)^2T^5, \\
    \textstyle \sum_{i=t}^t \sigma_{X_i}^2 &\geq b := \sigma_{X_1}^2 = 64C^2(G+H\|w_1\|_2)^2.
\end{align*}
Recall that $\kappa = 1+DH/G$, so
\begin{align*}
    \frac{B}{b} = \frac{(G+DH)^2T^5}{(G+H\|w_1\|_2)^2} \leq (\kappa T)^5.
\end{align*}
Also recall that $\Phi=\sqrt{\log(20dT\log(2\kappa T)/\delta)}$, so
\begin{align*}
    \sqrt{\log\frac{4dT}{\delta} + \log\log \frac{B}{b}} \leq \sqrt{\log \frac{20dT\log(\kappa T)}{\delta}} \leq \Phi.
\end{align*}
Therefore, with probability at least $1-\delta/2T$,
\begin{align}
    \|\weight_t \nabla\L(x_t) - g_t\|_2 
    &\leq C\Phi \sqrt{\sum_{i=1}^t [8Ci^2(G+H\|w_i-x_{i-1}\|_2)]^2} \notag\\
    &\leq 8C^2\Phi \sqrt{\sum_{i=1}^t i^4(G+H\|w_i-x_{i-1}\|_2)^2}.
    \label{eq:variance-concentration}
\end{align}
By union bound, with probability at least $1-\delta/2$,
\begin{align}
    &\sum_{t=1}^T \|\weight_t\nabla\L(x_t)-g_t\|_2\|w_t-x^*\|_2 \notag\\
    \leq & 8C^2\Phi \sum_{t=1}^T \sqrt{\sum_{i=1}^t i^4(G+H\|w_i-x_{i-1}\|_2)^2} \|w_t-x^*\|_2 \notag
    \intertext{We use the identity $(a+b)^2\leq 2a^2+2b^2$ and $\sqrt{a+b}\leq \sqrt{a}+\sqrt{b}$:}
    \leq & 8C^2\Phi \sum_{t=1}^T \left(\sqrt{\sum_{i=1}^t 2G^2i^4} + \sqrt{\sum_{i=1}^t 2H^2\|w_i-x_{i-1}\|_2^2i^4}\right) \|w_t-x^*\|_2 \label{eq:step-1}
\end{align}
\textbf{1.1.} We bound these two sums separately. For the first sum,
recall that $A=8\sqrt{2}C^2$, so
\begin{align*}
    8C^2\sum_{t=1}^T \sqrt{\sum_{i=1}^t 2G^2i^4}\|w_t-x^*\|_2 
    \leq AG\sum_{t=1}^T t^{5/2}(\|w_t\|_2+\|x^*\|_2).
\end{align*}

\textbf{1.2.}
For the second sum, we apply Young's inequality ($ab\leq \frac{1}{2\lambda}a^2+\frac{\lambda}{2}b^2$) for each $t$:
\begin{align}
    &8C^2\sum_{t=1}^T\sqrt{\sum_{i=1}^t2H^2\|w_i-x_{i-1}\|_2^2i^4}\|w_t-x^*\|_2 \notag\\
    \leq & AH\sum_{t=1}^T\frac{1}{2\lambda_t}\sum_{i=1}^t (\|w_i-x_{i-1}\|_2^2i^4) + \frac{\lambda_t}{2}\|w_t-x^*\|_2^2 \notag
    % line 1
    \intertext{We first bound $\|w_i-x_{i-1}\|_2^2 \leq 2\|w_i\|_2^2+2\|x_{i-1}\|_2^2$. Recall that $x_0=0$ and for $i\geq 2$, $x_{i-1}=\sum_{j=1}^{i-1} \frac{\weight_j}{\weight_{1:{i-1}}}w_j$, so $\|x_{i-1}\|_2^2\leq \sum_{j=1}^{i-1}\frac{\weight_j}{\weight_{1:i-1}}\|w_j\|_2^2$ by convexity. Consequently,}
    \leq & AH\sum_{t=1}^T \left( \frac{1}{\lambda_t}\sum_{i=1}^t i^4\|w_i\|_2^2 + \frac{1}{\lambda_t}\sum_{i=2}^t i^4\sum_{j=1}^{i-1} \frac{\weight_j\|w_j\|_2^2}{\weight_{1:i-1}} + \lambda_t(\|w_t\|_2^2+\|x^*\|_2^2) \right). \label{eq:step-1.2}
\end{align}
We define $\lambda_t = ct^{5/2}$ for some constant $c$ to be determined later, and we apply change of summation on the first two sums:
\begin{restatable}{lemma}{LemmaChangeSum}
\label{lem:change-sum}
For any sequence $a_i,b_j,c_k$,
\begin{align*}
    \sum_{i=1}^N a_i \sum_{j=1}^i b_j = \sum_{i=1}^N b_i \sum_{j=i}^N a_j,
    \quad \textit{and} \quad 
    \sum_{i=1}^N a_i \sum_{j=1}^i b_j \sum_{k=1}^j c_k = \sum_{i=1}^N c_i \sum_{j=i}^N a_j \sum_{k=i}^j b_k.
\end{align*}
\end{restatable}

\textbf{1.2.1.}  For the first summation,
\begin{align*}
    \sum_{t=1}^T \frac{1}{\lambda_t} \sum_{i=1}^t i^4\|w_i\|_2^2
    &= \sum_{t=1}^T\sum_{i=t}^T \frac{1}{\lambda_i} t^4\|w_t\|_2^2
    \intertext{For decreasing function $f$, $\sum_{i=t+1}^T f(i)\leq \int_t^T f(x)\,dx$, then:}
    &\leq \sum_{t=1}^T \left( \frac{1}{ct^{5/2}} + \int_t^\infty \frac{1}{cx^{5/2}}\,dx \right)t^4\|w_t\|_2^2
    % \leq \sum_{t=1}^T \left(\frac{2}{3c}t^{5/2} + \frac{1}{c}t^{3/2}\right) \|w_t\|_2^2. 
    \leq \sum_{t=1}^T \frac{5}{3c}t^{5/2}\|w_t\|_2^2.
\end{align*}
\textbf{1.2.2.} For the second term, by Proposition \ref{prop:alphas}, $\weight_{1:i-1}\geq (i-1)^4/4$, so
\begin{align*}
    \sum_{t=1}^T \frac{1}{\lambda_t} \sum_{i=2}^t i^4\sum_{j=1}^{i-1} \frac{\weight_j\|w_j\|_2^2}{\weight_{1:i-1}} 
    &\leq \sum_{t=1}^T \frac{1}{ct^{5/2}}\sum_{i=2}^{t} \frac{4i^4}{(i-1)^4} \sum_{j=1}^{i} j^3\|w_j\|_2^2 
    \intertext{For all $i\geq 2$, we can bound $i/(i-1)\leq 2$. We then apply change of summation, which gives:}
    &\leq \sum_{t=1}^T\sum_{i=t}^T\frac{1}{ci^{5/2}}\sum_{j=t}^i 64t^3\|w_t\|_2^2
    \leq \sum_{t=1}^T \frac{192}{c} t^{5/2}\|w_t\|_2^2.
\end{align*}
The last inequality is again derived from the integral bound:
\begin{align*}
    \sum_{i=t}^T \frac{1}{i^{5/2}}\sum_{j=t}^i 1 \leq \sum_{i=t}^T \frac{1}{i^{3/2}}
    \leq \frac{1}{t^{3/2}} + \int_t^\infty \frac{1}{x^{3/2}}\,dx
    \leq \frac{3}{t^{1/2}}.
\end{align*}
% We then bound the inner summation with the same argument using the integral inequality. By Proposition \ref{prop:alphas}, $\sum_{j=t}^i\frac{(j+1)^{2(k-1)}}{\weight_{1:j}}\leq \sum_{j=t}^i\frac{4(j+1)^4}{j^4} \leq 64i$. Hence,
% \begin{align*}
%     \sum_{i=t}^T\frac{c}{\lambda_i}\sum_{j=t}^i\frac{(j+1)^{2(k-1)}}{\weight_{1:j}}
%     &\leq \sum_{i=t}^T \frac{64}{i^{3/2}} \leq \frac{64}{t^{3/2}}+\int_t^\infty \frac{64}{x^{3/2}}\,dx = \frac{64}{t^{3/2}} + \frac{128}{t^{1/2}} \leq \frac{192}{t^{1/2}}.
% \end{align*}
% Consequently, the second term is bounded by:
% \begin{align*}
%     \sum_{t=1}^T\sum_{i=t}^T\frac{1}{\lambda_i}\sum_{j=t}^i\frac{(j+1)^{2(k-1)}}{\weight_{1:j}} 2\weight_t\|w_t\|_2^2
%     &\leq \sum_{t=1}^T \frac{384}{c}t^{5/2}\|w_t\|_2^2.
% \end{align*}
In conclusion, upon substituting \textbf{1.2.1.} and \textbf{1.2.2.} into \eqref{eq:step-1.2} and setting $c=14$, we get:
\begin{align*}
    &8C^2\sum_{t=1}^T\sqrt{\sum_{i=1}^t2H^2\|w_i-x_{i-1}\|_2^2i^4}\|w_t-x^*\|_2 \\
    \leq & AH \sum_{t=1}^T \left(\frac{5}{3c}t^{5/2}\|w_t\|_2^2 + \frac{192}{c}t^{5/2}\|w_t\|_2^2 + ct^{5/2}(\|w_t\|_2^2+\|x^*\|_2^2) \right) \\
    \leq & AH \sum_{t=1}^T 28 t^{5/2}(\|w_t\|_2^2+\|x^*\|_2^2).
\end{align*}
% The second inequality holds by choosing $c=14$.
Moreover, upon substituting \textbf{1.1.} and \textbf{1.2.} into \eqref{eq:step-1}, we get: with probability at least $1-\delta/2$,
\begin{align*}
    &\sum_{t=1}^T\|\weight_t\nabla\L(x_t)-g_t\|_2\|w_t-x^*\|_2 \\
    % \leq & \sum_{t=1}^T A\left( G\sqrt{\log(2d/\delta_t)} + 14H(1+\log(2d/\delta_t)) \right) t^{5/2} (\|w_t\|_2^2+\|x^*\|_2^2).
    \leq & A\Phi\sum_{t=1}^T  Gt^{5/2}(\|w_t\|_2+\|x^*\|_2) + 28Ht^{5/2} (\|w_t\|_2^2+\|x^*\|_2^2).
\end{align*}
% \begin{align*}
%     &\sum_{t=1}^T\|\weight_t\nabla\L(x_t)-g_t\|_2\|w_t-x^*\|_2 \\
%     \leq & 8C^2\sum_{t=1}^T G\sqrt{\log\frac{2d}{\delta_t}} t^{5/2}(1+\|w_t-x^*\|_2^2) + 40H\left(1+\log\frac{2d}{\delta_t}\right) t^{5/2} (\|w_t\|_2^2+\|x^*\|_2^2) \\
%     % \leq & 8C^2\sum_{t=1}^T \sqrt{\log\frac{2d}{\delta_t}}\left( Gt^{5/2} + (2G+28H)(\|w_t\|_2^2+\|x^*\|_2^2)t^{5/2} \right).
%     \leq & 8C^2G\sqrt{\log\frac{2d}{\delta_t}}T^{7/2} + 16C^2\sum_{t=1}^T\left(G\sqrt{\log\frac{2d}{\delta_t}}+20H(1+\log\frac{2d}{\delta_t})\right)t^{5/2}(\|w_t\|_2^2+\|x^*\|_2^2).
% \end{align*}

\textbf{Step 2:} We can bound $\sum_{t=1}^T\|\gamma_t\|_2\|w_t-x^*\|_2$ in a similar way. By Lemma \ref{lem:X-R-nSG} and definition of $\sigma_t$ in \eqref{eq:RDP-noise}, $R_i|R_{1:i-1}$ is mean-zero $\nSG(\sigma_{R_i})$, where
\begin{align*}
    \sigma_{R_i} 
    &= C\sqrt{d}\sigma_\D \sigma_i
    = \frac{8\sqrt{d}\sigma_\D C}{\rho}   \sqrt{\log_2(2T)}(G+H\max_{j\in[i]}\|w_j-x_{j-1}\|_2)i^2.
\end{align*}
Next, we can bound $\sum_{i\in I_t}\sigma_{R_i}^2$ as follows: for all $t$,
\begin{align*}
    \sum_{i\in I_t} \sigma_{R_i}^2 \geq \min_{i\in I_t}\sigma_{R_i}^2 \geq b:= \frac{64d\sigma_\D^2C^2}{\rho^2}\log_2(2T)G^2.
\end{align*}
On the other hand, since $|I_t|\leq \log_2(2T)$,
\begin{align}
    \sum_{i\in I_t} \sigma_{R_i}^2 \leq B_t:= \frac{64d\sigma_\D^2C^2}{\rho^2} \log_2^2(2T) (G+DH)^2 t^4.
    % \label{eq:parameter-pf-2}
    \notag
\end{align}
Hence, $B_t/b \leq \log_2(2T)\kappa^2T^4 \leq (2\kappa T)^5$ (because $\log_2(2T)\leq 2T$ and $\kappa\geq 1$). By definition of $\Phi$,
\begin{align*}
    \sqrt{\log\frac{4dT}{\delta}+\log\log\frac{B_t}{b}}
    \leq \sqrt{\log\frac{20dT\log(2\kappa T)}{\delta}} = \Phi.
\end{align*}
Recall that $A'=8\sqrt{d}\sigma_\D C^2$. By Lemma \ref{lem:nSG-concentration}, for each $t$, with probability at least $1-\delta/2T$,
\begin{align}
    \|\gamma_t\|_2 
    &\leq C\sqrt{\sum_{i\in I_t}\sigma_{R_i}^2\left(\log\frac{4dT}{\delta}+\log\log\frac{B_t}{b}\right)}
    % \intertext{Recall that $A'=8\sqrt{dK}\sigma_\D C^2$. Upon substituting \eqref{eq:parameter-pf-2}, we get:}
    \leq \frac{A'}{\rho}(G+DH)\Phi \log_2(2T)t^2.
    \label{eq:privacy-concentration}
\end{align}
By union bound, with probability at least $1-\delta/2$,
\begin{align*}
    \sum_{t=1}^T\|\gamma_t\|_2\|w_t-x^*\|_2 
    % &\leq \sum_{t=1}^TC\sqrt{\sum_{i\in I_t}\sigma_{R_i}^2\left(\log\frac{2dT}{\delta}+\log\log\frac{B_t}{b}\right)}\|w_t-x^*\|_2
    % &\leq \sum_{t=1}^T C\sqrt{B_t\left(\log\frac{2dT}{\delta}+\log\log\frac{B_t}{b}\right)}D
    % \intertext{Recall that $A'=8\sqrt{dK}\sigma_\D C^2, \Phi=\sqrt{\log(16dT\log(2\kappa T)/\delta)}$, so bounding $\sum_{i\in I_t}\sigma_{R_i}^2 \leq B_t$ gives:}
    % &\leq \sum_{t=1}^T \sqrt{\log\frac{16dT\log(2\kappa T)}{\delta}}\frac{A'}{\rho} (G+DH) \log_2(2T) t^2 \|w_t-x^*\|_2 \\
    &\leq \sum_{t=1}^T \frac{A'}{\rho}(G+DH)\Phi \log_2(2T)t^2(\|w_t\|_2+\|x^*\|_2).
\end{align*}

In conclusion, we take the union bound on the results from \textbf{step 1.} and \textbf{step 2.} and substitute it back to the starting point \eqref{eq:parameter-free-base}.  Then with probability at least $1-\delta$,
\begin{align*}
    \weight_{1:T}(\L(x_T)-\L(x^*))
    &\leq R_T(x^*) + \sum_{t=1}^T 28AH\Phi t^{5/2}(\|w_t\|_2^2+\|x^*\|_2^2)  \\
    & \quad \ + \sum_{t=1}^T \left(AG\Phi t^{5/2} + A'(G+DH)\frac{\Phi\log_2(2T)t^2}{\rho}\right)(\|w_t\|_2+\|x^*\|_2).
\end{align*}
Define $\xi_t,\nu_t$ as in the theorem, and recall that $\weight_{1:T}\geq T^4/4$. This completes the proof.
\end{proof}

\LemmaChangeSum*

\begin{proof}
The proof is basically re-pairing the summations:
\begin{align*}
    \sum_{i=1}^N\sum_{j=1}^i a_i b_j 
    &= a_1b_1 + a_2(b_1+b_2) + a_3(b_1+b_2+b_3) + \dotsm \\
    &= (a_1+\dotsm+a_N)b_1 + (a_2+\dotsm+a_N)b_2 + \dotsm \sum_{i=1}^T\sum_{j=0}^{T-i} a_{t-j} b_i.
\end{align*}
For the second part of the theorem, denote $B_j^i=\sum_{k=j}^i b_k$. By first part,
\begin{align*}
    LHS &= \sum_{i=1}^N a_i \sum_{j=1}^i c_j \sum_{k=j}^i b_k \\
    &= \sum_{i=1}^N\sum_{j=1}^i a_ic_j (B_j^N-B_{i+1}^N) \\
    &= \sum_{i=1}^N\sum_{j=i}^N a_jc_iB_i^N - a_jc_iB_{j+1}^N \\
    &= \sum_{i=1}^N\sum_{j=i}^Na_jc_iB_i^j.
\end{align*}
We then recover the lemma once we substitute $B_i^j=\sum_{k=i}^j b_k$.
\end{proof}

% \begin{lemma}
% \label{lem:sub-gaussian-sum}
% Suppose $\P[\|R_t\|_*\geq \epsilon]\leq 2\exp(-\epsilon^2/2\sigma_t^2)$ and $\sigma_t^2\leq \sigma_{t+1}^2$ for all $t$. Then with probability $\geq 1-\delta$, for all $t$,
% \begin{align*}
%     \sum_{i\in I_t}\|R_i\|_*^2 \leq 2\log(2t)\sigma^2 \log\left(\frac{2T\log(2t)}{\delta}\right).
% \end{align*}
% \end{lemma}

% \begin{proof}
% By pigeonhole principle and union bound, for each $t$,
% \begin{align*}
%     \P\left[\sum_{i\in I_t} \|R_i\|_*^2 \geq \epsilon \right] 
%     &\leq \P\left[ \exists i\in I_t,\, \|R_i\|_*^2 \geq \frac{\epsilon}{|I_t|} \right] \\
%     &\leq \sum_{[a,b]\in I_t} \P\left[ \|r_{[a:b]}\|_*^2 \geq \frac{\epsilon}{|I|} \right] \\
%     &\leq 2\log(2t) \exp\left(\frac{-\epsilon}{2\log(2t)\sigma^2}\right).
% \end{align*}
% Equivalently, for each $t$, with probability at least $1-\delta$, 
% \begin{align*}
%     \sum_{[a:b]\in I_t}\|r_{[a:b]}\|_*^2 \leq 2\log(2t)\sigma^2 \log\left(\frac{2\log(2t)}{\delta}\right).
% \end{align*}
% After replacing $\delta$ with $\delta/T$ and taking the union bound over all $t$, with probability at least $1-\delta$,
% \begin{align*}
%     \forall t, \, \sum_{[a:b]\in I_t}\|r_{[a:b]}\|_*^2 \leq 2\log(2t)\sigma^2 \log\left(\frac{2T\log(2t)}{\delta}\right).
% \end{align*}
% \end{proof}

\end{document}